\documentclass[twoside,11pt]{article}

\usepackage[nohyperref]{jmlr2e}
\usepackage{lastpage}
\usepackage{mathtools}
\usepackage{empheq}
\usepackage{multirow}
\usepackage{bbm}
\usepackage{booktabs}
\usepackage{xcolor}
\usepackage{xurl}
\usepackage{algorithm}
\usepackage{algpseudocode}
\usepackage{hyperref}

\hypersetup{hidelinks}
\graphicspath{{Figures/}}

\newcommand{\dd}{\mathrm{d}}

\newcommand{\E}{\mathbb{E}}

\newtheorem{monotoneremark}{Remark}

\algblock{Input}{EndInput}
\algnotext{EndInput}
\algblock{Output}{EndOutput}
\algnotext{EndOutput}

\newcommand{\s}{\pi^{(\lambda)}}
\newcommand{\sh}{\hat{\pi}^{(\lambda)}}
\newcommand{\tpi}{\theta^{\pi^{(\lambda)}}}
\newcommand{\thpi}{\theta^{\hat{\pi}^{(\lambda)}}}
\newcommand{\ts}{\theta^{\hat{\pi}^{\vartheta_a}}}

\ShortHeadings{Adaptive Reparametrized Time for Diffusion Sampling}{Huang, Tang and Zhou}
\firstpageno{1}

\title{Adaptive Reparametrized Time for Score-Based Diffusion Sampling}
\author{
\name Yilie Huang\thanks{Corresponding author.} \email yilie.huang@polyu.edu.hk\\
\addr Department of Applied Mathematics\\
The Hong Kong Polytechnic University\\
Hung Hom, Kowloon, Hong Kong
\AND
\name Wenpin Tang \email wt2319@columbia.edu\\
\addr Department of Industrial Engineering and Operations Research\\
Columbia University\\
New York, NY 10027, USA
\AND
\name Xun Yu Zhou \email xz2574@columbia.edu\\
\addr Department of Industrial Engineering and Operations Research\\
Data Science Institute\\
Columbia University\\
New York, NY 10027, USA
}

\begin{document}

\editor{Zhaoran Wang}
\maketitle\thispagestyle{plain}
\makeatletter
\begingroup
  \long\def\@makefntext#1{\noindent#1}
  \footnotetext{Accepted for publication in the \textit{Journal of Machine Learning Research} (JMLR).}
\endgroup
\makeatother

\begin{abstract}%
We study timestep allocation for score-based diffusion sampling, where the learned reverse-time dynamics is discretized on a finite grid. Uniform and hand-crafted schedules are standard, but their fixed, {\it ad hoc} prescriptions can be suboptimal. We propose Adaptive Reparameterized Time (ART), a continuous-time control formulation that learns a time change by treating sampling-clock speed as the control, so a uniform grid on the learned clock induces adaptive timesteps in diffusion time and a leading-order Euler error surrogate defines the allocation objective. To solve the resulting high-dimensional deterministic control problem, we introduce ART-RL, an auxiliary randomized Gaussian-policy formulation that turns schedule learning into continuous-time reinforcement learning. We prove ART-RL is equivalent to ART at optimality: the mean of the optimal Gaussian policy solves ART. We further establish policy evaluation and policy improvement characterizations and derive trajectory-based moment identities that yield implementable actor--critic updates. Experiments ranging from controlled low-dimensional settings to image generation show that ART-learned schedules, when plugged into existing samplers by changing only the timestep grid, consistently improve sample quality over strong baselines at matched budgets. The learned schedules also exhibit broad generalizability, transferring without retraining across sampling budgets, data sets, solvers, pipelines, and representation spaces.
\end{abstract}


\section{Introduction}

Diffusion models \citep{Ho20, Song19, song2020score} generate samples by transforming noise into data, thereby producing draws from a target distribution learned from examples.
They now underpin a broad range of modern generative systems, including text-to-image models such as DALL·E 2 \citep{Ramesh22} and Stable Diffusion \citep{Rombach2022}, text-to-video generators such as Sora \citep{Sora2024}, Make-A-Video \citep{Singer2022} and Veo \citep{Veo2024} and, more recently, diffusion-based large language models such as Mercury \citep{KK25} and LLaDA \citep{Nie25}.

A diffusion pipeline typically separates {\it training} from {\it sampling}: the score or denoising model is learned during pretraining, and inference generates samples by running a reverse-time dynamics with a numerical discretization.
This paper focuses on the sampling stage, where one must choose a finite set of timesteps to discretize the learned reverse-time process.
Because each step requires evaluating the learned model, the choice of time grids directly dictates how a fixed computational budget is spent and can substantially affect both efficiency and sample quality.
Most existing approaches adopt uniform grids or hand-crafted schedules \citep{DDIM, song2020score, karras2022elucidating, Chen2023,lu2022dpm}, but these choices are rarely derived from a {\it principled} optimization framework.
Our goal is to provide a control-theoretic framework and approach that treat timestep selection as a systematic design problem for diffusion sampling.

The main contributions of this paper are summarized below:
\begin{itemize}
\item
{\em Methodology}:
We formulate timestep allocation for diffusion sampling as a continuous-time optimal control  problem, termed {\em Adaptive Reparameterized Time} (ART).
ART introduces a time change called the sampling clock and models the local progression in diffusion time as a controllable rate, which reallocates function evaluations along the reverse sampling trajectory while respecting a fixed overall time budget.
To solve the resulting control problem, which is inherently in high-dimensional state spaces in most applications including image generation, we develop {\em ART-RL}, a continuous-time reinforcement learning (CTRL) approach  that learns this rate via randomized, {\it Gaussian} policies and actor--critic iterations, leveraging recent theoretical advances in CTRL \citep{wang2020reinforcement, jia2021policy, jia2021policypg}.

\item
{\em Theory}:
We establish a rigorous link between ART and its randomized counterpart ART-RL.
First, we show that the auxiliary randomized formulation is not merely a relaxation: it aligns with the original deterministic ART objective in that the mean of the optimal ART-RL Gaussian policy solves the ART control problem.
Second, we develop continuous-time actor--critic theory specialized to time reparameterization, including characterizations of policy evaluation and policy improvement that yield explicit, implementable update rules.
These results lead to moment conditions for both the critic and the actor and provide improvement guarantees that underpin the resulting learning algorithm for the optimal time schedule.

\item
{\em Experiments}:
We evaluate ART across low- and high-dimensional settings, multiple numerical solvers, sampling pipelines, representation spaces, and a broad range of sampling budgets.
In controlled experiments with an analytical score model, in MNIST with a deliberately small score network, and in the standard EDM pipeline for CIFAR--10, ART consistently improves over Uniform, DPM, and EDM schedules at matched evaluation budgets, including the largest budgets where the hand-designed EDM schedule is the strongest.
See a quick overview of these empirical gains in Figure~\ref{fig:quant-overview}.
All comparisons keep the trained score model, network backbone, solver, and sampling protocol fixed; so the gains are purely from our choice of  timestep allocation.
This makes ART a principled  schedule-learning method rather than an {\it ad hoc} one  tied to a particular architecture, sampler, or diffusion pipeline.

\item
{\em Transfer/Generalization}:
Our experiments show that the schedule trained on CIFAR--10 under a given number of time step budget transfers directly, without retraining, across timestep counts, data sets, sampling pipelines, and representation spaces.
It outperforms the benchmarks not only on AFHQv2, FFHQ, and ImageNet--64 under the pixel-space EDM pipeline, but also on ImageNet--512 under EDM2 which simultaneously involves a modern backbone and sampling pipeline, a latent-space representation, and high-resolution image generation.
These results demonstrate that ART-RL learns a reusable timestep schedule rather than a data-set-specific tuning artifact; its one-time training cost can therefore be amortized beyond the particular data set, solver, and pipeline on which it is learned.
\end{itemize}

To our best knowledge, this is the first work that develops a  control theory based framework for learning timestep schedules in generative diffusion sampling, providing a theoretically grounded alternative to the existing fixed heuristic grids.
The proposed ART-RL method is purely data-driven and learns a reusable schedule that improves both direct sampling performance and transfer performance.

\medskip
\noindent
\emph{Relevant literature.}
Diffusion models were first developed in discrete time, including DDPM \citep{Ho20} and DDIM \citep{DDIM}.
The continuous-time viewpoint of \citet{song2020score} recasts diffusion modeling through a stochastic differential equation (SDE) formulation, unifying and extending earlier discrete constructions.
On the inference side, many samplers can be interpreted as numerical methods for the learned reverse-time dynamics, including the predictor-corrector scheme of \citet{song2020score}, exponential-integrator style approaches \citep{ZC23}, and higher-order solvers \citep{ZSC23, WCW24}.
In addition, convergence analyses for diffusion inference under uniform or hand-crafted discretizations are studied in \citet{LLT22, Chen2023, ChenChe23, LW23, Ben24, LYl24, HHL25}. Learned timestep grids are also studied by \citet{tong2025learning,pei2025optimal}, using teacher-based optimization and dynamic programming, respectively.

Continuous-time reinforcement learning (CTRL) was introduced and formulated by \citet{wang2020reinforcement} as entropy-regularized stochastic control in continuous time and spaces, where exploration is represented by {\it relaxed}  (randomized) controls, formalizing the trial-and-error mechanism central to reinforcement learning.
Following \citet{wang2020reinforcement}, a sequence of works have built a model-free theory for CTRL through martingale-based analyses \citep{jia2021policy, jia2021policypg, JZ22, TZq24}, complemented by explicit performance guarantees \citep{huang2025sublinear,huang2025data}. A related study also investigates policy optimization in the continuous-time setting \citep{ZTY23,ZTY26}. CTRL theory has been applied to financial portfolio selection \citep{huang2024mean,dai2023data} and asset--liability management \citep{huang2025continuous}, as well as to fine-tuning generative AI diffusion models \citep{GZZ24, ZC24, ZC25} and image inpainting \citep{huang2026aid}. In particular, \citet{dai2023data} exploit the special structure of the Merton problem with power utility and introduce a family of Gaussian policies without considering entropy regularization. They then prove the mean of the optimal Gaussian policy solves the original problem. This idea indeed inspires the formulation of ART-RL in this paper, even though the application domain and problem setting are very different here.

\medskip
\noindent
This paper is the full version of our NeurIPS 2026 conference paper \citep{huang2026conference}. It substantially extends the theoretical foundations of ART-RL with new theorems on policy evaluation and policy improvement, and trajectory-based moment characterizations that underpin implementable, data-driven actor--critic learning. It also provides a more extensive numerical analysis of sampling performance and transfer across data sets, solvers, and diffusion pipelines.

\medskip
\noindent
\emph{Organization of the paper.}
Section~\ref{sc2} reviews score-based diffusion and probability flow ODE sampling.
Section~\ref{sec:art} introduces ART as a time-reparameterization control formulation.
Section~\ref{sc3} presents ART-RL as a randomized auxiliary problem, with a provable connection to ART.
Section~\ref{sec_algorithm} develops the theory and an actor--critic algorithm for solving ART-RL, including policy evaluation and policy improvement.
Section~\ref{sec:experiments} reports the experimental results.
Section~\ref{sc6} concludes. Additional numerical results are placed in the appendix.

\section{Revisiting Continuous-Time Score-Based Diffusion Models}
\label{sc2}

We briefly revisit continuous-time score-based diffusion models for generative AI (GenAI) along with key notations; see \citet{TZSS25} for a detailed exposition.
A diffusion model consists of a forward diffusion process that progressively corrupts data with an unknown target distribution over a physical time interval $\tau \in [0,T]$, driving the distribution toward a simple reference law (e.g. Gaussian), and a backward generative process that transports samples from this reference back toward the original target  distribution (Figure~\ref{fig:diffusion-schematic}).

\begin{figure}[tbp]
    \centering
    \includegraphics[width=0.95\columnwidth]{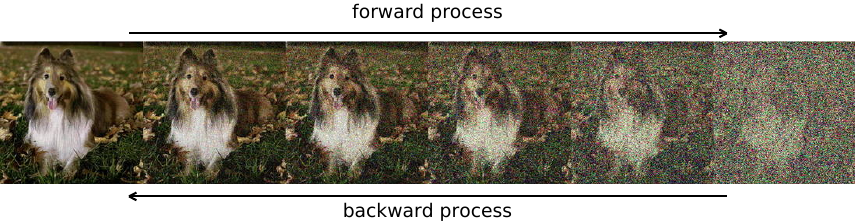}
    \caption{Illustration of the forward noising process and the corresponding backward generative process in a score-based diffusion model.}
    \label{fig:diffusion-schematic}
\end{figure}

\emph{Forward diffusion.}
The forward dynamics are given by the It\^o SDE
\begin{equation}
    \mathrm{d}\,\bar{x}(\tau)
    = -\,f\bigl(\tau \bigr) \bar{x}(\tau) \,\mathrm{d}\tau
      + g(\tau)\,\mathrm{d}w(\tau),
    \qquad \tau\in[0,T],\quad \bar{x}(0)\sim p_0 \in \mathcal{P}(\mathbb{R}^d),
    \label{eq:forward-sde}
\end{equation}
where $w=\{w(\tau):\tau\in[0,T]\}$ is a standard Wiener process (Brownian motion) in $\mathbb{R}^d$, $f:[0,T]\times\mathbb{R}^d\to\mathbb{R}^d$ and $g:[0,T]\to\mathbb{R}_+$ are measurable coefficients, $\mathcal{P}(\mathbb{R}^d)$ denotes the set of Borel probability measures on $\mathbb{R}^d$, and $p_0$ is the unknown target distribution.  Here $d$ is usually a very large number in a typical task such as image generation. Let $p_\tau$ be the law of $\bar{x}(\tau)$ and write its score function as $S(\tau,x)=\nabla_x\log p_\tau(x)$. Under standard well-posedness assumptions for common choices of $(f,g)$, the SDE \eqref{eq:forward-sde} maps $p_0$ along the family $\{p_\tau\}_{\tau\in[0,T]}$ toward a tractable reference distribution at time $T$.

\emph{Backward sampling.}
For sampling, one can work with the reverse-time diffusion or, equivalently, with a deterministic probability flow ordinary differential equation (ODE) that shares the same marginals as the reverse-time SDE.
As shown by \citet[Theorem 5.1]{TZSS25}, the ODE and the reverse SDE yield the same family $\{p_\tau\}$. Denoting the backward state by $\tilde{x}(\tau) \coloneqq \bar{x}(T-\tau)$ with initialization $\tilde{x}(0)\sim p_T$ and using a trained score model $\hat{S}(\tau,x)$ in place of the unknown $S(\tau,x)$, the implementable backward probability flow ODE is
\begin{equation}
    \frac{\mathrm{d}\tilde{x}(\tau)}{\mathrm{d}\tau}
    = f\bigl(T-\tau\bigr) \tilde{x}(\tau)
      + \frac{1}{2}\,g\bigl(T-\tau\bigr)^2\,
        \hat{S}\bigl(T-\tau,\tilde{x}(\tau)\bigr),
    \qquad \tau\in[0,T],\quad \tilde{x}(0)\sim p_T.
    \label{eq:backward-pf}
\end{equation}

\emph{Euler discretization.}
To generate samples, we numerically integrate \eqref{eq:backward-pf} on a grid $0=\tau_0<\tau_1<\cdots<\tau_K=T$ with step sizes $h_i=\tau_{i+1}-\tau_i$, and denote $\tilde{x}_i \coloneqq \tilde{x}(\tau_i)$. Using the explicit Euler method, we obtain
\begin{equation}
    \tilde{x}_{i+1}
    = \tilde{x}_i
      + h_i\!\left[
          f\bigl(T-\tau_i\bigr) \tilde{x}_i
          + \tfrac{1}{2}\,g\bigl(T-\tau_i\bigr)^2\,
            \hat{S}\bigl(T-\tau_i,\tilde{x}_i\bigr)
        \right],
    \qquad i=0,\ldots,K-1,\quad \tilde{x}(0)\sim p_T.
    \label{eq:euler}
\end{equation}
A uniform grid $\tau_i=iT/K$ is simple to implement and widely used, but such a single and na\"ive global step size cannot capture the inevitable variation in numerical characteristics  along the trajectory.
When $K$ is small, sampling is computationally efficient but discretization error can be significant.
When $K$ is large, under a uniform grid the additional function evaluations are spread evenly rather than concentrated on where dominant errors can be most effectively reduced, in which case a nontrivial fraction of the extra computation is poorly utilized.
Intuitively, early stages of the reverse process, where samples are close to noise, may tolerate coarser resolution, whereas later stages typically benefit from finer steps.
These considerations motivate {\it adaptive}, data-driven time discretizations that redistribute steps under a fixed total time budget $T$, allocating computation to where it has the greatest impact on accuracy and sample quality.

\section{ART: Time Reparameterization as Control}\label{sec:art}

We now formulate adaptive timestep selection as a continuous-time control problem.
The central idea is to introduce a reparameterized sampling clock and to treat the progression of physical diffusion time as a controlled process.
By allowing the sampling trajectory to advance at variable speeds at different epochs, this formulation strategically redistributes computational effort under a fixed total time budget.
In this section, we first describe the reparameterized dynamics induced by the time change, and subsequently introduce an objective function that formalizes optimal timestep allocation.

\subsection{Time Reparameterization and Controlled Dynamics}\label{subsec:art-dynamics}

Rather than evolving the reverse process on a fixed physical-time grid, we introduce a reparameterized clock that governs how diffusion time is traversed during sampling.
This auxiliary time variable decouples numerical resolution from the original diffusion horizon and enables adaptive redistribution of function evaluations.
Specifically, let $\psi:[0,T]\to\mathbb{R}$ be a continuous mapping from the reparameterized time $t$ to the physical diffusion time $\tau$, so that $\tau=\psi(t)$ with $\psi(0)=0$ and $\psi(T)=T$.
Figure~\ref{fig:art-time-change} depicts this correspondence between the two clocks and the induced nonuniform physical-time grid.

\begin{figure}[t]
    \centering
    \includegraphics[width=\linewidth]{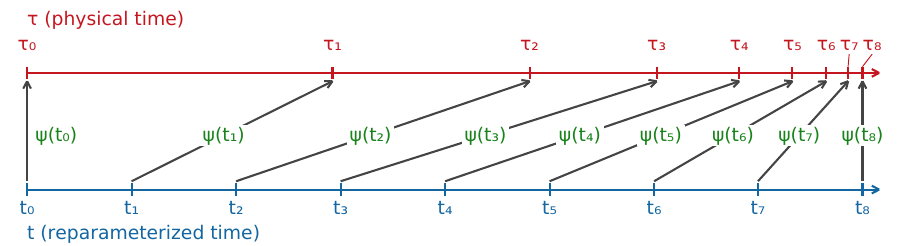}
    \caption{ART as a time change between two clocks. The physical diffusion time $\tau$ (top) and the reparameterized time $t$ (bottom) are linked by a continuous map $\tau=\psi(t)$ with $\psi(0)=0$ and $\psi(T)=T$. A uniform grid in $t$ induces a generally nonuniform grid in $\tau$, with local speed $\theta(t)=\dot{\psi}(t)$ controlling how function evaluations are redistributed along the trajectory.}
    \label{fig:art-time-change}
\end{figure}

On the reparameterized clock, the reverse-time state is represented as $x(t)\coloneqq \tilde{x}(\psi(t))$ with initialization $x(0)\sim p_T$.
The time-change is characterized by the control $\theta(t)\coloneqq \dot{\psi}(t)$, the instantaneous rate at which diffusion time advances relative to the new clock.
The boundary conditions on $\psi$ translates into an integral  constraint $\int_0^T \theta(t)\,\mathrm{d}t = \psi(T) - \psi(0) = T$, ensuring that the full diffusion horizon is traversed over the sampling interval.

Importantly, we do not restrict $\psi$ to be monotone a priori, consequently allowing $\theta(t)$ to take either sign.
This modeling choice is deliberate: it avoids prematurely excluding admissible control trajectories that may arise when learning $\theta$ in a data-driven manner.
The case when monotone time reparameterization must be respected is studied in Appendix~\ref{app:monotone-art}.

If the reparameterized time is discretized uniformly as $0=t_0<t_1<\cdots<t_K=T$, the resulting physical-time grid is given by $\tau_i=\psi(t_i)$, with step sizes that generally vary across $i$.
From a numerical perspective, the trajectory $x(\cdot)$ evolves on the new clock, while the control $\theta(\cdot)$ determines where progression in physical diffusion time is accelerated or slowed down.
This mechanism enables the sampler to allocate resolution adaptively along the reverse trajectory, placing finer discretization where it is most beneficial.
We refer to this time-reparameterized sampling framework as \emph{Adaptive Reparameterized Time} (ART).

We now formulate the  controlled dynamics  under ART.
Since the reparameterized state is defined by $x(t)=\tilde{x}(\psi(t))$, the evolution of $x$ on the new clock follows directly from the chain rule. 
Taking $\psi$ as  another state variable and $\theta$ as the control variable, the state dynamics is
\begin{subequations}\label{eq_original_dynamics}
\begin{empheq}[left=\empheqlbrace]{align}
\dot{x}(t) &= \theta(t)\,F\bigl(x(t),\psi(t)\bigr),
\qquad x(0)\sim p_T, \label{eq:reparam-dyn-x}\\
\dot{\psi}(t) &= \theta(t),
\qquad \psi(0)=0,\ \psi(T)=T, \label{eq:reparam-dyn-psi}
\end{empheq}
\end{subequations}
where $F$ is the backward probability-flow vector field evaluated at the physical time $T-\psi$, namely
\begin{equation}\label{eq:F-def}
F(x,\psi)\coloneqq f\bigl(T-\psi\bigr)\,x
+\frac{1}{2}\,g\bigl(T-\psi\bigr)^{2}\,\hat{S}\bigl(T-\psi, x\bigr),
\end{equation}
with $f$ and $g$ being the coefficients of the forward diffusion, and $\hat{S}$  a learned score function.
The equation \eqref{eq:reparam-dyn-psi} simply records that $\psi$ accumulates at rate $\theta$; so $\theta$ can be interpreted as a local time-scaling factor under the new clock.
Moreover, we have the time budget constraint
\begin{equation}\label{eq:theta-budget}
\int_{0}^{T}\theta(t)\,\mathrm{d}t \;= T,
\end{equation}
which formalizes that the sampler must allocate a total amount $T$ of physical-time progression across the interval $t\in[0,T]$.
This constraint underlines an important feature of ART: any local deceleration of the dynamics, corresponding to smaller values of $\theta(t)$ and hence finer resolution in a given region, must be compensated by acceleration elsewhere, implying that improvements in numerical accuracy are achieved not by increasing the overall computational budget but by redistributing it along the trajectory in a strategic manner.

\subsection{Euler Error Surrogate and Control Objective}\label{subsec:art-objective}

To motivate the formulation of an appropriate objective for selecting the time-warping rate $\theta$ on the $t$-clock, we quantify how the Euler discretization behaves under the controlled dynamics \eqref{eq:reparam-dyn-x}.
The basic principle is that the leading-order one-step error is governed by the local stiffness indicator of the probability-flow dynamics, and hence can be used as a proxy for where additional resolution is most valuable.
We proceed in the same spirit as the Euler discretization in \eqref{eq:euler}, but now on a fixed,  generic step $[t_i,t_{i+1})$ of the $t$-clock with stepsize $h_i\coloneqq t_{i+1}-t_i$, where the implementation uses a {\it constant} control value per step, denoted by $\theta_i$.

Let $E_i$ denote the one-step Euler residual:
\[
E_i \coloneqq x(t_{i+1})-\Bigl(x(t_i)+h_i\,\theta_i\,F\bigl(x(t_i),\psi(t_i)\bigr)\Bigr).
\]
A second-order Taylor expansion of the solution map around $(x(t_i),\psi(t_i))$ yields the local error
\begin{equation}\label{eq:ltesummary}
E_i=\frac{h_i^2}{2}\,\theta_i^2\,Q\bigl(x(t_i),\psi(t_i)\bigr)+O(h_i^3),
\end{equation}
where the coefficient $Q$ collects terms arising from differentiating the probability-flow field along the trajectory.
More explicitly,
\begin{equation}\label{eq:Qexplicit}
\begin{aligned}
Q(x, \psi) =&\left[f(T-\psi) I_d+\frac{1}{2}(g(T-\psi))^2 \nabla_{x} \hat{S}(T-\psi, x)\right]\\
&\quad{}\times\left[f(T-\psi) x+\frac{1}{2}(g(T-\psi))^2 \hat{S}(T-\psi, x)\right] \\
& -f^{\prime}(T-\psi) x-g(T-\psi) g^{\prime}(T-\psi) \hat{S}(T-\psi, x)-\frac{1}{2}(g(T-\psi))^2 \frac{\partial \hat{S}(T-\psi, x)}{\partial\tau}.
\end{aligned}
\end{equation}
Equation~\eqref{eq:ltesummary} shows that the second-order error (in the step size) of the Euler scheme is quadratic in $\theta_i$, modulated by the term $Q(x,\psi)$ evaluated along the trajectory.
This function $Q$ captures local geometric and model-induced stiffness of the probability-flow field, and the control $\theta$ determines how strongly this stiffness ought to be felt on a given step.
As a result, regions where $|Q(x,\psi)|$ is large are precisely where aggressive time progression would amplify discretization error and hence one needs to proceed slowly by taking a small $\theta$, and vice versa.
Implementation details for evaluating $Q$ are given in Appendix~\ref{app:implementation}.

This motivates interpreting $|Q(x,\psi)|\,\theta(t)^2$ as a local cost density that guides advancing time progression based on $Q(x,\psi)$ adaptively. Together with the time budget constraint \eqref{eq:theta-budget}, we hence introduce the following objective functional
\begin{equation}\label{eq_original_objective}
\begin{aligned}
J^\theta&(s,y,\phi)
=\E\Big[\int_s^T\!\bigl(-|Q(x(t),\psi(t))|\theta^2(t)-\gamma\theta(t)\bigr)\,\dd t
+\gamma T \,\Bigm|\, x(s)=y,\ \psi(s)=\phi \Big],
\end{aligned}
\end{equation}
where $\gamma\in\mathbb{R}$ is the Lagrange multiplier for \eqref{eq:theta-budget}.

We define the optimal value function associated with this control problem as
\begin{equation}\label{eq_original_value_function}
V(s,y,\phi)\coloneqq\sup_{\theta=\theta(\cdot)}J^\theta(s,y,\phi).
\end{equation}

Consequently,  ART reframes timestep allocation as the problem of controlling the time-warping rate $\theta$ in the augmented dynamics \eqref{eq_original_dynamics}, with the objective \eqref{eq_original_objective} capturing how numerical error should be managed  along the reverse diffusion trajectory.

\section{Randomized Control and Reinforcement Learning Formulation}
\label{sc3}

The formulation in Section~\ref{sec:art} provides a principled way to pose timestep allocation as a control problem, but it does not immediately yield a practical solution method in high dimensions.
In general, the ART objective \eqref{eq_original_objective} admits no closed-form solutions, and the associated Hamilton--Jacobi--Bellman (HJB) equation with a large $d$ 
is numerically prohibitive due to the curse of dimensionality.
To remedy this, we introduce an auxiliary randomized control reformulation in which the time-warping rate is produced by a stochastic policy.
The role of randomization here is not for ``exploration" due to an unknown environment in the usual reinforcement learning (RL) sense; rather, it is a technical device that enables us to apply the recently developed continuous-time RL theory and algorithms \citep{wang2020reinforcement, jia2021policy, jia2021policypg}.
We refer this reformulation as \emph{Adaptive Reparameterized Time via Reinforcement Learning} (ART-RL), which we develop in the remainder of this section.

\subsection{ART-RL: An Auxiliary Problem with Gaussian Policies}
\label{sec:art-rl}

We replace the deterministic control $\theta$ by a randomized feedback policy that assigns, at each $(t,x,\psi)$, a probability distribution generating time-warping rates.
Specifically, we take the following Gaussian policy class.\footnote{The reason for choosing the Gaussian policies will be revealed in the subsequent theoretical analysis.} Its variance depends on the local numerical sensitivity encoded by $Q$.
\begin{equation}\label{eq:gaussian_policy}
\s(\cdot \mid t,x,\psi)
=\mathcal{N}\!\left(\mu(t,x,\psi),\,\frac{\lambda}{\lvert Q(x,\psi)\rvert}\right),
\end{equation}
where $\mu$ is a (deterministic) measurable function and $\lambda\ge 0$ is a scalar parameter.
The particular form of the variance ties policy randomization  to the geometry of the surrogate error: since $|Q|$ governs the Euler residual through \eqref{eq:ltesummary}, the variance $\lambda/|Q|$ suppresses policy-induced randomness in stiff regions (large $|Q|$) while allowing comparatively more randomness otherwise, while $\lambda$ controls the overall level of this randomization without changing the mean.
For analysis, we assume $|Q(x,\psi)|>0$ almost surely on compact intervals of $(x,\psi)$.
In implementation  we replace \(\lvert Q\rvert\) by \(\lvert Q\rvert\vee\varepsilon:= \max(|Q|, \varepsilon)\) for a small \(\varepsilon>0\).

We now present the ``exploratory formulation" of ART, following \citet{wang2020reinforcement}.
Denote by $\Pi^{(\lambda)}$  the collection of policies of the form \eqref{eq:gaussian_policy}.
For a fixed policy $\s\in\Pi^{(\lambda)}$, the corresponding state
$(x^{\s}(t),\psi^{\s}(t))_{t\in[0,T]}$ 
satisfies the ``exploratory" dynamics
\begin{subequations}\label{eq_auxiliary_dynamics}
\begin{align}
\frac{\dd x^{\s}(t)}{\dd t}
&= \int_{\mathbb{R}} \theta\, F\bigl(x^{\s}(t), \psi^{\s}(t)\bigr)\, \notag\\[-2pt]
&\qquad \s\!\left(\theta \mid t, x^{\s}(t), \psi^{\s}(t)\right)\,\dd\theta,
\quad x^{\s}(0)=x_0 \sim p_T,
\label{eq_auxiliary_dynamics_x}\\
\frac{\dd \psi^{\s}(t)}{\dd t}
&= \int_{\mathbb{R}} \theta\, \s\!\left(\theta \mid t, x^{\s}(t), \psi^{\s}(t)\right)\,\dd\theta,
\quad \psi^{\s}(0)=0,\ \ \psi^{\s}(T)=T.
\label{eq_auxiliary_dynamics_psi}
\end{align}
\end{subequations}
Moreover, the performance criterion is
\begin{equation}\label{eq_auxiliary_objective}
\begin{aligned}
J^{\s}(s,y,\phi)
= \E\Bigg[\, &\int_{s}^{T}\!\int_{\mathbb{R}}
\Bigl(-\,\lvert Q\bigl(t,x^{\s}(t),\psi^{\s}(t)\bigr)\rvert\,\theta^{2}
\;-\; \gamma\,\theta\Bigr)\\
&\qquad{}\times \s(\theta\mid t,x^{\s}(t),\psi^{\s}(t))\,\dd\theta\,\dd t \\
&\qquad\qquad +\; (\gamma + \lambda) \,T  \ \Bigm| \ x^{\s}(s)=y,\ \psi^{\s}(s)=\phi \Bigg].
\end{aligned}
\end{equation}
Here, the additional term $\lambda T$ in \eqref{eq_auxiliary_objective} is to compensate a constant bias induced by Gaussian randomization, so that the resulting criterion is comparable to the deterministic one under the same mean control.
To see this, fix any mean function $\mu(t,x,\psi)$ and consider the Gaussian policy \eqref{eq:gaussian_policy}.
Taking the policy expectation of the deterministic running cost yields the identity
\(
\int_{\mathbb R}\!\bigl(-\,|Q|\,\theta^2-\gamma\,\theta\bigr)\,\s(\theta\mid t,x,\psi)\,d\theta
= -\,|Q|\,\mu^2-\gamma\,\mu-\lambda.
\)
The associated optimal value function is
\begin{equation}\label{eq_auxiliary_value_function}
V^{(\lambda)}(s,y,\phi) \;=\; \max_{\s\in\Pi^{(\lambda)}} J^{\s}(s,y,\phi),
\qquad (s,y,\phi)\in [0,T]\times\mathbb{R}^{d}\times\mathbb{R}.
\end{equation}

\subsection{Connecting the Original and Randomized Formulations}\label{subsec:art-vs-art-rl}

We now establish that the solution to the original ART control problem \eqref{eq_original_value_function} can be recovered from that  to the randomized one \eqref{eq_auxiliary_value_function}. Indeed, the value function $V$ of \eqref{eq_original_value_function} satisfies the HJB equation
\begin{equation}\label{eq_original_hjb}
V_t + \sup_{\theta}\Bigl\{ \bigl(V_x^\top F(x,\psi)+V_{\psi}-\gamma\bigr)\theta - |Q(x,\psi)|\,\theta^2 \Bigr\}=0,
\end{equation}
together with the terminal condition $V(T,x,\psi)=\gamma T$.
Meanwhile,
the value function $V^{(\lambda)}$ of \eqref{eq_auxiliary_value_function} satisfies
\begin{equation}\label{eq_auxiliary_hjb}
V^{(\lambda)}_t + \sup_{\mu}\Bigl\{ \bigl(V^{(\lambda)\top}_x F(x,\psi)+V^{(\lambda)}_{\psi}-\gamma\bigr)\mu - |Q(x,\psi)|\Bigl(\mu^2+\frac{\lambda}{|Q(x,\psi)|}\Bigr) \Bigr\}=0,
\end{equation}
with terminal condition $V^{(\lambda)}(T,x,\psi)=(\gamma+\lambda)T$.

The following theorem discloses a precise relationship between them.

\begin{theorem}
\label{thm_relationship}
If \(V\) is a classical solution to the HJB equation \eqref{eq_original_hjb}, then $V^{(\lambda)}$ is a classical solution to the HJB equation \eqref{eq_auxiliary_hjb} where
\begin{equation}
\label{eq_auxiliary_value_function_2}
V^{(\lambda)}(t,x,\psi) = V(t,x,\psi) + \lambda t
\end{equation}
is a classical solution to HJB equation \eqref{eq_auxiliary_hjb}. Moreover,
\begin{equation}
\label{eq_auxiliary_optimal_policy}
\pi^{(\lambda)*}(\cdot | t,x,\psi) = \mathcal{N}\left(\mu^*(t,x,\psi), \frac{\lambda}{|Q(x, \psi)|}\right) \quad \text{with } \mu^*(t,x,\psi) = \frac{V_x^\top F(x, \psi) + V_{\psi} -\gamma}{2|Q(x,\psi)|}
\end{equation}
is the optimal policy for the auxiliary problem \eqref{eq_auxiliary_value_function} subject to the dynamics \eqref{eq_auxiliary_dynamics}. Finally, \(\mu^*(t,x,\psi)\) is the optimal policy for the original problem \eqref{eq_original_value_function} subject to the dynamics \eqref{eq_original_dynamics}.
\end{theorem}

\begin{proof}

First of all, 
the ``sup" in the two equations \eqref{eq_original_hjb} and \eqref{eq_auxiliary_hjb}  are respectively achieved at
\[
\theta^*(t,x,\psi)= \frac{V^{\top}_x F(x, \psi) + V_{\psi} -\gamma}{2|Q(x,\psi)|},\;\;\mu^*(t,x,\psi) = \frac{V^{(\lambda)\top}_x F(x, \psi) + V^{(\lambda)}_{\psi} -\gamma}{2|Q(x,\psi)|}.
\]
So if $V$ solves \eqref{eq_original_hjb} and $V^{(\lambda)}$ is chosen to satisfy \eqref{eq_auxiliary_value_function_2}, then the above two maximizers are identical. Moreover,
it is straightforward to check that \(V^{(\lambda)}\) solves  \eqref{eq_auxiliary_hjb}.

Next, we show \(V^{(\lambda)}\) and \(\s\) are respectively the optimal value function and optimal policy for the randomized  problem \eqref{eq_auxiliary_value_function} via a standard verification approach. Fix a policy \(\s\). 
Applying It\^o's lemma to \(V^{(\lambda)}(t, x^{\s}(t), \psi^{\s}(t))\), we have
\[
\begin{aligned}
& V^{(\lambda)}(T, x^{\s}(T), \psi^{\s}(T)) - V^{(\lambda)}(s, x^{\s}(s), \psi^{\s}(s)) \\
&+ \int_s^T \biggl(-|Q(x^{\s}(t), \psi^{\s}(t))| \mu(t, x^{\s}(t), \psi^{\s}(t))^2 - \lambda - \gamma \mu(t, x^{\s}(t), \psi^{\s}(t)) \biggr) \dd t\\
=& \int_s^T \biggl( V^{(\lambda)}_t + (V^{(\lambda)\top}_x F(x^{\s}(t), \psi^{\s}(t)) + V^{(\lambda)}_{\psi} - \gamma) \mu(t, x^{\s}(t), \psi^{\s}(t)) \\
&- |Q(x^{\s}(t), \psi^{\s}(t))| \mu(t, x^{\s}(t), \psi^{\s}(t))^2 -\lambda \biggr) \dd t\\
\leq& 0,
\end{aligned}
\]
where the last inequality follows from the HJB equation \eqref{eq_auxiliary_hjb}.
Thus, we have
\[
\begin{aligned}
&V^{(\lambda)}(s,y,\phi) \\
\geq& \E \biggl[ \int_s^T \biggl(-|Q(x^{\s}(t), \psi^{\s}(t))| \mu(t, x^{\s}(t), \psi^{\s}(t))^2 - \lambda - \gamma \mu(t, x^{\s}(t), \psi^{\s}(t)) \biggr) \dd t \\
&+ V^{(\lambda)}(T, x^{\s}(T), \psi^{\s}(T)) | x^{\s}(s)=y, \psi^{\s}(s)=\phi \biggr]\\
=& J^{\s}(s,y,\phi).
\end{aligned}
\]
When the policy \eqref{eq_auxiliary_optimal_policy} is taken, the above inequality becomes equality because  \eqref{eq_auxiliary_optimal_policy} achieves the supremum in the HJB equation \eqref{eq_auxiliary_hjb}. This establishes the optimality of the policy \eqref{eq_auxiliary_optimal_policy} along with $V^{(\lambda)}$ being the optimal value function.  On the other hand, noticing the previous analysis applies to the case when \(\lambda=0\) and \(\mu^*\) is independent of \(\lambda\), we arrive at the final conclusion of theorem.

\end{proof}


Theorem~\ref{thm_relationship} implies  that the ART solution can be recovered by solving the ART-RL problem \eqref{eq_auxiliary_value_function} with Gaussian policies. The latter can indeed be solved using an actor--critic scheme that is {\it not} directly applicable to the former. 
We will carry this out in the next section.

\begin{monotoneremark}[Time change]\label{rem:monotone-art}
In our current formulation, ART accommodates either direction of motion along the diffusion clock. If a nondecreasing time change is desired, one can impose $\theta\ge0$ in \eqref{eq_original_value_function} and $\mu\ge0$ in the Gaussian policies \eqref{eq:gaussian_policy}. An analogous correspondence holds for these constrained problems: the values differ by $\lambda s$, and the mean of an optimal Gaussian policy solves the constrained ART problem, with the converse also holding. See Appendix~\ref{app:monotone-art} and Theorem~\ref{thm:monotone-equivalence} for details.
\end{monotoneremark}

\section{ART-RL Actor--Critic: Theory and Algorithm}\label{sec_algorithm}
\label{sc4}

Building on Theorem~\ref{thm_relationship}, we now work within ART-RL to learn the ART optimizer. 
Our approach is premised upon the continuous-time actor--critic framework of \citet{jia2021policypg}, adapted to the ART-RL setting and in particular the Gaussian policies.
We start with two theorems on policy evaluation and policy improvement, followed by development of the resulting algorithm.

\subsection{Theoretical Results: Policy Evaluation and Improvement}\label{sec_theory}

In an actor--critic method, the critic estimates the value of a given policy, while the actor updates/improves the policy by moving in a favorable direction guided  by this value information.
Theorem~\ref{thm_policy_improvement} below formalizes this idea. For any Gaussian policy \(\pi^{(\lambda)}\), it first characterizes the associated value function, and  then constructs a new Gaussian policy \(\Tilde{\pi}^{(\lambda)}\), whose mean is obtained by a Hamiltonian-type maximization based on that value function, and shows that this updated policy improves the value for all states. 

\begin{theorem}
\label{thm_policy_improvement}
(i) The value function under a Gaussian policy
\[
\pi^{(\lambda)}(\cdot | t,x,\psi) = \mathcal{N}(\mu(t,x,\psi), \frac{\lambda}{|Q(x, \psi)|})
\]
is given by
\begin{equation}
\label{eq_thm2_j}
J^{\s}(t,x,\psi) = \Bar{v}(t,x,\psi) + \lambda t,
\end{equation}
where \(\Bar{v}\) satisfies the linear PDE
\begin{equation}
\label{eq_thm2_pde}
\Bar{v}_t + (\Bar{v}_x^\top F(x,\psi) +  \Bar{v}_{\psi} - \gamma) \mu(t,x,\psi) - |Q(x,\psi)| \mu(t,x,\psi)^2 = 0,
\end{equation}
with the terminal condition \(\Bar{v}(T,x,\psi) = \gamma T\).

(ii) Consider the policy defined as
\begin{equation}
\label{eq_thm2_policy}
\begin{aligned}
\Tilde{\pi}^{(\lambda)}(\cdot | t,x,\psi) &= \mathcal{N}\left(\Tilde{\mu}(t,x,\psi), \frac{\lambda}{|Q(x, \psi)|}\right),\\
\Tilde{\mu}(t,x,\psi) &= \frac{\Bar{v}_x(t,x,\psi)^{\top} F(x,\psi) + \Bar{v}_{\psi}(t,x,\psi) - \gamma}{2|Q(x,\psi)|}.
\end{aligned}
\end{equation}
Then \(\Tilde{\pi}^{(\lambda)}\) improves the original policy \(\pi^{(\lambda)}\) in the sense that
\[
J^{\Tilde{\pi}^{(\lambda)}}(t,x,\psi) \geq J^{\pi^{(\lambda)}}(t,x,\psi) \quad \text{for all} \quad (t,x,\psi).
\]
\end{theorem}

\begin{proof}
(i) For a given policy \(\s\), applying the Feynman-Kac formula leads to the following PDE for the value function \(J^{\s}\):
\[
J_t^{\s} + (J_x^{\s \top} F(x,\psi) + J_{\psi}^{\s} - \gamma) \mu(t,x,\psi) - |Q(x,\psi)| \mu(t,x,\psi)^2 - \lambda = 0,
\]
with \(J^{\s}(T,x,\psi) = (\gamma + \lambda) T\).
Since the function \(J^{\s}\) defined in \eqref{eq_thm2_j} satisfies the above PDE,  the result follows from the uniqueness of the solution to linear PDE.

(ii) From Part (i), \(J^{\Tilde{\pi}^{(\lambda)}}\) can also be expressed as \(J^{\Tilde{\pi}^{(\lambda)}}(t,x,\psi)=\Tilde{v}(t,x,\psi) + \lambda t\), where \(\Tilde{v}\) satisfies
\begin{equation}
\label{eq_thm2_pde2}
\Tilde{v}_t + (\Tilde{v}_x^\top F(x,\psi) + \Tilde{v}_{\psi} - \gamma) \Tilde{\mu}(t,x,\psi) - |Q(x,\psi)| \Tilde{\mu}(t,x,\psi)^2 = 0,
\end{equation}
with \(\Tilde{v}(T,x,\psi) = \gamma T\).


Take the left hand side of \eqref{eq_thm2_pde} as a quadratic function in $\mu(t,x,\psi)$, which clearly achieves the maximum at $\Tilde{\mu}(t,x,\psi)$.
Hence
\begin{equation}
\label{eq_thm2_pde3}
\begin{aligned}
&\Bar{v}_t + (\Bar{v}_x^\top F(x,\psi) + \Bar{v}_{\psi} - \gamma) \Tilde{\mu}(t,x,\psi) - |Q(x,\psi)| \Tilde{\mu}(t,x,\psi)^2 \\
\geq &\Bar{v}_t + (\Bar{v}_x^\top F(x,\psi) +  \Bar{v}_{\psi} - \gamma) \mu(t,x,\psi) - |Q(x,\psi)| \mu(t,x,\psi)^2 = 0.
\end{aligned}
\end{equation}
It now follows from the comparison principle for PDEs applied to \eqref{eq_thm2_pde2} and \eqref{eq_thm2_pde3} that \(\Tilde{v}(t,x,\psi) \geq \Bar{v}(t,x,\psi)\). Equivalently,  \(J^{\Tilde{\pi}^{(\lambda)}}(t,x,\psi) \geq J^{\pi^{(\lambda)}}(t,x,\psi) \text{ for all } (t,x,\psi)\).
\end{proof}

Theorem~\ref{thm_policy_improvement} 
is a pure theoretical result that cannot be used directly for computation, because it involves solving PDEs that are intractable in high dimensions. However, it provides the foundation for the next theorem that in turn underpins an implementable, data-driven actor--critic scheme.

\begin{theorem}
\label{thm_martingale}
(i) Let \(\s\) be a Gaussian policy and \(\hat{V}\) be a continuous function with
\[
\hat{V}(T, x, \psi) = \gamma T + \lambda T.
\]
Denote by \(\tpi\) a control sampled from \(\s\) and by \((x^{\tpi}, \psi^{\tpi})\)  the state process under  \(\tpi\) with the initial condition \(x^{\tpi}(s) = y\) and \(\psi^{\tpi}(s) = \phi\). If for any measurable function \(\xi\) and every \((s, y, \phi) \in ([0, T] \times \mathbb{R}^d \times \mathbb{R})\),
\[
\begin{aligned}
\mathbb{E}\biggl[\int_s^T &\xi(t, x^{\tpi}(t), \psi^{\tpi}(t)) \, \\
&\biggl( \mathrm{d} \hat{V}(t, x^{\tpi}(t), \psi^{\tpi}(t)) - (|Q(x^{\tpi}(t), \psi^{\tpi}(t))| \theta^{\pi^{(\lambda)}2}(t) + \gamma \tpi(t)) \dd t \biggr) \biggr] = 0
\end{aligned}
\]
holds, then \(\hat{V} \equiv J^{\s}\). 

(ii) Let  \(\hat{\pi}(\cdot | t, x, \psi) = \mathcal{N}( \hat{\mu}(t, x, \psi), \frac{\lambda}{|Q(x, \psi)|} )\) where \(\hat{\mu}\) is a continuous function.
Denote by \(\thpi\) a control sampled from \(\sh\) and by \((x^{\thpi}, \psi^{\thpi})\)  the state process under  \(\thpi\) with the initial condition \(x^{\thpi}(s) = y\) and \(\psi^{\thpi}(s) = \phi\).
If for any measurable function \(\eta\) and for every \((s, y, \phi) \in ([0, T] \times \mathbb{R}^d \times \mathbb{R})\),
\[
\begin{aligned}
\mathbb{E}\biggl[\int_s^T & \eta(t, x^{\thpi}(t), \psi^{\thpi}(t)) \biggl( \thpi(t) - \hat{\mu}(t, x^{\thpi}(t), \psi^{\thpi}(t)) \biggr) \, \\
&\biggl( \mathrm{d} J^{\s}(t, x^{\thpi}(t), \psi^{\thpi}(t))  -(|Q(x^{\thpi}(t), \psi^{\thpi}(t))| \theta^{\Tilde{\pi}^{(\lambda)}2}(t)\\
&\qquad{} + \gamma \thpi(t)) \dd t \biggr) \biggr] = 0
\end{aligned}
\]
holds, then \(\hat{\mu} \equiv \Tilde{\mu}\) as defined in Theorem \ref{thm_policy_improvement}-(ii).
\end{theorem}

\begin{proof}
(i) The equation presented in the statement is the martingale orthogonality condition for policy evaluation, developed in \citet{jia2021policy}. By following the same reasoning as in the proof of Proposition 4 therein, we arrive at the following expression:
\[
\begin{aligned}
\hat{V}(s,y,\phi) &= \E \biggl[\int_s^T \biggl(-\gamma \tpi(t) - |Q(x^{\tpi}(t), \psi^{\tpi}(t))| \theta^{\pi^{(\lambda)}2}(t) \biggr) \dd t\\
&\qquad{}+\gamma T+\lambda T \Big| x^{\tpi}(s)=y,\ \psi^{\tpi}(s)=\phi \biggr],
\end{aligned}
\]
which is consistent with the definition of the value function \(J^{\s}\).

(ii) To simplify the notation, we denote \[
\begin{gathered}
\eta(t):=\eta(t, x^{\thpi}(t), \psi^{\thpi}(t)), \qquad
\Bar{v}(t) = \Bar{v}(t, x^{\thpi}(t), \psi^{\thpi}(t)),\\
F(t):=F(x^{\tpi}(t), \psi^{\tpi}(t)), \qquad
Q(t):=Q(x^{\tpi}(t), \psi^{\tpi}(t)),\\
\hat{\mu}(t):= \hat{\mu}(t, x^{\thpi}(t), \psi^{\thpi}(t)),
\end{gathered}
\]
and
\(
J(t) :=  J^{\s}(t, x^{\thpi}(t), \psi^{\thpi}(t)) = \Bar{v}(t) + \lambda t.
\)

Applying Ito's lemma to \(J\), we have
\[
\begin{aligned}
0=& \E\biggl[\int_s^T \eta(t) (\thpi(t) - \hat{\mu}(t)) \biggl\{ \dd J(t) - (\gamma \thpi(t) + |Q(t)| \theta^{\hat{\pi}^{(\lambda)}2}) \dd t \biggr\}\biggr]\\
=&\E\biggl[\int_s^T \eta(t) (\thpi(t) - \hat{\mu}(t)) \biggl\{ \Bar{v}_t(t)+\lambda\\
&\qquad{}+(\Bar{v}_x(t)F(t)+\Bar{v}_{\psi}(t)-\gamma)\thpi(t)-|Q(t)|\theta^{\hat{\pi}^{(\lambda)}2}\biggr\}\dd t\biggr]\\
=& \E\int_s^T \eta(t) \frac{\lambda}{|Q(t)|} \biggl\{ \Bar{v}_x(t)F(t) + \Bar{v}_{\psi}(t) - \gamma - 2|Q(t)| \hat{\mu}(t) \biggr\} \dd t.\\
\end{aligned}
\]
Because this equations holds for any \(\eta\), the integrand must be zero.
Therefore, we obtain the condition
\[
\Bar{v}_x(t)F(t) + \Bar{v}_{\psi}(t) - \gamma - 2|Q(t)| \hat{\mu}(t) = 0,
\]
or
\[
\hat{\mu}(t) = \frac{\Bar{v}_x(t)F(t) + \Bar{v}_{\psi}(t) - \gamma}{2|Q(t)|}.
\]
This expression is identical  with \(\Tilde{\mu}\) defined in \eqref{eq_thm2_policy} of Theorem \ref{thm_policy_improvement}-(ii). 
\end{proof}


The term $\thpi(t) - \hat{\mu}(t, x^{\thpi}(t), \psi^{\thpi}(t))$ in the equation of Theorem~\ref{thm_martingale}-\textnormal{(ii)} reveals why invoking stochastic policies is vital for our approach to work: $\thpi(t)$ is sampled from the Gaussian policy with the mean $ \hat{\mu}(t, x^{\thpi}(t), \psi^{\thpi}(t))$ so the two terms are generally {\it different}. In this case the equation provides a genuine direction for policy improvements. If we consider only deterministic policies, then these terms are identical and the equation becomes trivial giving away no information at all about the direction for improvement.  This observation highlights the necessity of recasting ART as ART-RL. In the following subsections, we develop an ART-RL algorithm based on the established results.

\subsection{Actor--Critic Parameterization and Update Rules}\label{sec_alg}

Guided by Theorems~\ref{thm_policy_improvement} and \ref{thm_martingale}, we now turn the analytical results into concrete actor--critic update rules. The parametrization of the critic follows the structure \eqref{eq_thm2_j} in Theorem~\ref{thm_policy_improvement}\textnormal{(i)}, while the parametrization of the actor is chosen within the Gaussian class \eqref{eq:gaussian_policy} designed for ART-RL. Specifically,
For function approximations of the actor and critic, by Theorems~\ref{thm_policy_improvement} and \ref{thm_martingale}, we parameterize the value function and policy using two separate neural network functions \(NN^{\vartheta_c}\) and \(NN^{\vartheta_a}\):

\begin{equation}
\label{eq_parameterization}
\hat{V}^{\vartheta_c}(t,x,\psi) = NN^{\vartheta_c}(t,x,\psi) + \lambda t,
\qquad
\hat{\pi}^{\vartheta_a}(\cdot \mid t,x,\psi)
= \mathcal{N}\!\left(NN^{\vartheta_a}(t,x,\psi), \frac{\lambda}{|Q(x,\psi)|}\right).
\end{equation}

Applying Theorem~\ref{thm_martingale} to the parametrization \eqref{eq_parameterization} yields, for suitable test processes \(\xi\) and \(\eta\), the coupled moment conditions
\begin{equation}
\label{eq_moment_conditions}
\left\{
\begin{aligned}
&\mathbb{E}\biggl[ \int_0^T \xi(t)
\bigl( \dd \hat{V}^{\vartheta_c}(t, x^{\ts}(t), \psi^{\ts}(t))\\
&\qquad - \bigl(|Q(x^{\ts}(t), \psi^{\ts}(t))| \theta^{\hat{\pi}^{\vartheta_a}2}(t)
+ \gamma \ts(t)\bigr) \dd t \bigr) \biggr] = 0, \\
&\mathbb{E}\biggl[ \int_0^T \eta(t)
\bigl(\ts(t) - NN^{\vartheta_a}(t,x^{\ts}(t),\psi^{\ts}(t))\bigr)\\
&\qquad \bigl( \dd \hat{V}^{\vartheta_c}(t, x^{\ts}(t), \psi^{\ts}(t)) \\
&\qquad - \bigl(|Q(x^{\ts}(t), \psi^{\ts}(t))| \theta^{\hat{\pi}^{\vartheta_a}2}(t)
+ \gamma \ts(t)\bigr) \dd t \bigr) \biggr] = 0,
\end{aligned}
\right.
\end{equation}
where \((x^{\ts},\psi^{\ts})\) denotes the state process under a  control \(\ts\)
generated from the Gaussian policy \(\hat{\pi}^{\vartheta_a}\), which can be simulated and therefore observable as data.

In addition, 
we take the following test processes: 
\[
\xi(t) = \frac{\partial}{\partial \vartheta_c} NN^{\vartheta_c}(t, x^{\ts}(t), \psi^{\ts}(t)),
\qquad
\eta(t) = \frac{\partial}{\partial \vartheta_a} NN^{\vartheta_a}(t, x^{\ts}(t), \psi^{\ts}(t)),
\]
which are consistent with the standard choices in the RL actor--critic literature \citep{sutton1998reinforcement, konda1999actor,jia2021policypg, huang2025sublinear}.

To derive update rules, we interpret the moment conditions \eqref{eq_moment_conditions} as equations in \((\vartheta_c,\vartheta_a)\) and solve them by stochastic approximation. We use subscript \(n\) to denote quantities at iteration \(n\); for example, \(\vartheta_{c,n}\) denotes the value of \(\vartheta_c\) at the \(n\)-th iteration. 
Given the \(n\)-th observed trajectory \((x_n,\psi_n,\theta_n)\) generated/sampled from the current policy \(\hat{\pi}^{\vartheta_{a,n}}\) and a learning rate \(a_n>0\), we update the critic and actor parameters by

\begin{subequations}\label{eq_updating_rules}
\begin{align}
\label{eq_critic_update}
\vartheta_{c,n+1} \leftarrow\ &\vartheta_{c,n}
+ a_n \int_0^T
\frac{\partial NN^{\vartheta_c}}{\partial \vartheta_c}(t,x_n(t),\psi_n(t)) \nonumber\\
&\qquad\qquad\cdot \bigl[ \dd \hat{V}^{\vartheta_c}(t, x_n(t), \psi_n(t))
- (|Q(x_n(t), \psi_n(t))|\theta_n(t)^2 + \gamma \theta_n(t)) \dd t \bigr], \\[0.5em]
\label{eq_actor_update}
\vartheta_{a,n+1} \leftarrow\ &\vartheta_{a,n}
+ a_n \int_0^T
\frac{\partial NN^{\vartheta_a}}{\partial \vartheta_a}(t,x_n(t),\psi_n(t)) \nonumber\\
&\qquad\qquad\cdot \bigl(\theta_n(t)-NN^{\vartheta_a}(t,x_n(t),\psi_n(t))\bigr) \nonumber\\
&\qquad\qquad\cdot \bigl[ \dd \hat{V}^{\vartheta_c}(t, x_n(t), \psi_n(t))
- (|Q(x_n(t), \psi_n(t))|\theta_n(t)^2 + \gamma \theta_n(t)) \dd t \bigr].
\end{align}
\end{subequations}

Finally, the Lagrange multiplier \(\gamma\), enforcing the terminal constraint on \(\psi\), is updated along the same trajectory by
\begin{equation}
\label{eq_updating_alpha}
\gamma_{n+1} \leftarrow \gamma_{n} + a_n \bigl( \psi_n(T) - T \bigr).
\end{equation}
Equations~\eqref{eq_updating_rules} and \eqref{eq_updating_alpha} constitute the theoretical update rules of all the learnable parameters in ART-RL. 

\subsection{Time-Discretized ART-RL Actor--Critic Algorithm}\label{sec:pseudocode}

To have an implementable algorithm, the final step is to discretize the previous update rules. To this end, we work on a uniform time grid
\(0 = t_0 < t_1 < \dots < t_K = T\) with step size
\(\Delta t = T/K\). For the \(n\)-th iteration, write
\[
\hat{V}^{\vartheta_{c,n}}_{k}
:= \hat{V}^{\vartheta_{c,n}}(t_k, x_n(t_k), \psi_n(t_k))
= NN^{\vartheta_{c,n}}(t_k, x_n(t_k), \psi_n(t_k)) + \lambda t_k,
\qquad k = 0,\dots,K.
\]
A simple Riemann approximation of the integrals in
\eqref{eq_updating_rules} leads to the time-discretized critic and actor
updates
\begin{subequations}\label{eq_disc_updating_rules}
\begin{align}
\label{eq_disc_critic_update}
\vartheta_{c,n+1}
\leftarrow\ &\vartheta_{c,n}
+ a_n \sum_{k=0}^{K-1}
\frac{\partial NN^{\vartheta_{c,n}}}{\partial \vartheta_c}
\bigl(t_k, x_n(t_k), \psi_n(t_k)\bigr) \nonumber\\
&\quad \times \Bigl[
\hat{V}^{\vartheta_{c,n}}_{k+1} - \hat{V}^{\vartheta_{c,n}}_{k}
- \bigl(|Q(x_n(t_k), \psi_n(t_k))|\,\theta_n(t_k)^2
       + \gamma_n\,\theta_n(t_k)\bigr)\,\Delta t
\Bigr], \\[0.5em]
\label{eq_disc_actor_update}
\vartheta_{a,n+1}
\leftarrow\ &\vartheta_{a,n}
+ a_n \sum_{k=0}^{K-1}
\frac{\partial NN^{\vartheta_{a,n}}}{\partial \vartheta_a}
\bigl(t_k, x_n(t_k), \psi_n(t_k)\bigr) \nonumber\\
&\quad \times
\bigl(\theta_n(t_k) - NN^{\vartheta_{a,n}}(t_k, x_n(t_k), \psi_n(t_k))\bigr) \nonumber\\
&\quad \times \Bigl[
\hat{V}^{\vartheta_{c,n}}_{k+1} - \hat{V}^{\vartheta_{c,n}}_{k}
- \bigl(|Q(x_n(t_k), \psi_n(t_k))|\,\theta_n(t_k)^2
       + \gamma_n\,\theta_n(t_k)\bigr)\,\Delta t
\Bigr].
\end{align}
\end{subequations}
The update for the Lagrange multiplier is unchanged:
\begin{equation}
\label{eq_disc_updating_alpha}
\gamma_{n+1} \leftarrow \gamma_{n} + a_n \bigl(\psi_n(T) - T\bigr),
\end{equation}
where \(\psi_n(T) = \psi_n(t_K)\). The update uses each training rollout's terminal residual; the deployment grid is normalized as in Section~\ref{subsec:oned}.

We summarize the resulting ART-RL actor-critic scheme in
Algorithm~\ref{alg:art-rl-actor-critic}. The inner loop generates one trajectory
under the current Gaussian policy \eqref{eq_parameterization}, and  the outer loop
then updates the actor, critic, and Lagrange multiplier via
\eqref{eq_disc_updating_rules} and \eqref{eq_disc_updating_alpha}.

\begin{algorithm}[H]
\caption{Time-Discretized ART-RL Actor--Critic}\label{alg:art-rl-actor-critic}
\begin{algorithmic}
\For{$n = 1$ to $N$}
    \State Set $k = 0$, $t_0 = 0$, initialize $(x_n(t_0), \psi_n(t_0))$
    \While{$k < K$}
        \State Compute policy mean $m_{n,k} = NN^{\vartheta_{a,n}}(t_k, x_n(t_k), \psi_n(t_k))$
        \State Set $q_{n,k} = |Q(x_n(t_k),\psi_n(t_k))|\vee\varepsilon$
        \State Sample $\theta_n(t_k)\sim\mathcal N(m_{n,k},\lambda/q_{n,k})$
        \State $x_n(t_{k+1})=x_n(t_k)+\Delta t\,\theta_n(t_k)F(x_n(t_k),\psi_n(t_k))$
        \State $\psi_n(t_{k+1})=\psi_n(t_k)+\Delta t\,\theta_n(t_k)$
        \State Increment time: $t_{k+1} = t_k + \Delta t$, $k \leftarrow k+1$
    \EndWhile
    \State Collect trajectory $\{(t_k, x_n(t_k), \psi_n(t_k), \theta_n(t_k))\}_{k=0}^{K-1}$
    \State Update critic parameters $\vartheta_{c,n+1}$ via \eqref{eq_disc_critic_update}
    \State Update actor parameters $\vartheta_{a,n+1}$ via \eqref{eq_disc_actor_update}
    \State Update multiplier $\gamma_{n+1}$ via \eqref{eq_disc_updating_alpha}
\EndFor
\end{algorithmic}
\end{algorithm}

\begin{monotoneremark}[Direct optimization]\label{rem:Direct-optimization}
Direct grid optimization (DGO) parameterizes a fixed time grid by positive increments normalized to sum to $T$ and uses Adam to optimize these grid parameters under the same Euler-error surrogate as ART-RL. In the matched comparison in Appendix~\ref{app:dgo}, ART-RL yields better sampling results, suggesting that its randomized formulation and actor--critic learning can improve on direct optimization of the grid.
\end{monotoneremark}

\section{Numerical Experiments}\label{sec:experiments}

We now numerically evaluate ART-RL across several regimes that differ in dimensionality, numerical solver, model capacity, and experimental protocol.
The central practical question is whether ART-RL can improve existing samplers by changing only the timestep grid, while leaving the pretrained model, solver, and other sampling components unchanged. In this section,
the EDM- and EDM2-based experiments test this question in modern image-generation pipelines, while the one-dimensional analytical-score experiment isolates discretization effects and the MNIST experiment tests a less optimized score model.

\subsection{Experimental Setup and Baselines}\label{sec:exp-setup}
\emph{Data sets.}
We consider both synthetic and real-image settings.
For the former, we conduct an experiment where a synthetic target distribution on \(\mathbb{R}\) whose score function is known precisely and explicitly, allowing us to isolate the effect of time reparameterization from score-estimation errors.
For the latter, which consists of several high-dimensional image generation tasks, we work on the following data sets: 1) CIFAR-10 \citep{krizhevsky2009learning}, a data set of \(32\times 32\) natural images from ten classes; 2) AFHQv2, a variant of the AFHQ animal faces data set \citep{choi2020starganv2} with \(64\times 64\) images of cats, dogs, and wildlife; 3) FFHQ \citep{karras2019style}, a collection of human face images that we downsample to \(64\times 64\) as in \citet{karras2022elucidating}; and 4) ImageNet \citep{russakovsky2015imagenet}, which we evaluate at both \(64\times 64\) and \(512\times 512\) resolutions.
The ImageNet--64 experiment follows the standard EDM setup of \citet{karras2022elucidating}, while
the ImageNet--512 experiment adopts the EDM2 setting of \citet{karras2024analyzing}. The latter is important because it evaluates ART-RL on a more modern backbone and sampling pipeline, moves from pixel-space diffusion to latent-space diffusion, and tests substantially higher-resolution image generation.
To study a small-model regime, we also consider MNIST \citep{lecun2002gradient}, a data set of \(28\times 28\) grayscale handwritten digits, and train a compact score model directly on this data set.

\emph{Timestep schedules and baselines.}
We compare four timestep schedules that differ in how they allocate a fixed number of function evaluations along the reverse trajectory.

\emph{Uniform} is the simplest choice and serves as a reference baseline: it discretizes the physical time interval $\tau \in [0,T]$ using an equally spaced grid.

\emph{EDM} is the hand-designed schedule by \citet{karras2022elucidating}, which is widely adopted in diffusion sampling and is known to perform strongly on standard image benchmarks.
The discrete timesteps are calculated by
\[
\tau_k
=\Bigl(\sigma_{\max}^{1/\rho}
+\frac{k}{K}\bigl(\sigma_{\min}^{1/\rho}-\sigma_{\max}^{1/\rho}\bigr)\Bigr)^{\rho},
\qquad k=0,\ldots,K,
\]
with hyperparameters $\sigma_{\min}>0$, $\sigma_{\max}>\sigma_{\min}$, and $\rho>0$. Following \citet{karras2022elucidating} we use the recommended default $\rho=7$.
Equivalently, this construction corresponds to using a uniform grid in the transformed coordinate $\sigma^{1/\rho}$ between $\sigma_{\max}^{1/\rho}$ and $\sigma_{\min}^{1/\rho}$, and hence can be interpreted as a fixed, pre-specified time reparameterization.

\emph{DPM-Solver} \citep{lu2022dpm}, henceforth denoted as DPM, provides another strong hand-designed timestep grid that is widely used in fast diffusion sampling.
We use its standard uniform log-SNR grid only as a timestep schedule, without changing the numerical integrator.
In the variance-exploding (VE) setting, this gives the geometrically spaced noise levels
\[
\tau_k
=\sigma_{\max}\Bigl(\frac{\sigma_{\min}}{\sigma_{\max}}\Bigr)^{k/K},
\qquad k=0,\ldots,K,
\]
where \(\sigma_{\min}>0\) and \(\sigma_{\max}>\sigma_{\min}\).

\emph{ART-RL} is our learned schedule, obtained from the ART objective and implemented by Algorithm~\ref{alg:art-rl-actor-critic}.
With ART, a control $\theta$ induces a time change $\psi$, and sampling is performed by placing a uniform grid on the reparameterized clock and mapping it back to physical time through $\psi$.
This construction indeed includes the other three schedules as special cases: the identity map $\psi(t)=t$ recovers Uniform, selecting $\psi$ to match the EDM coordinate $\sigma^{1/\rho}$ (up to a constant rescaling) reproduces EDM, and selecting $\psi$ so that the induced grid is uniform in the DPM-Solver log-SNR coordinate leads to DPM.
Importantly, ART-RL learns \(\psi\) from data, making the schedule adaptive and allowing timestep allocation to move beyond hand-crafted designs when this improves sampling accuracy under a fixed evaluation budget.

\emph{Evaluation metrics.}
We assess sampling performance using metrics appropriate to the dimensionality and application settings.
In the one-dimensional synthetic experiment, sampling accuracy is quantified by the squared Wasserstein distance $W_2$ between the empirical distribution of generated samples and the (known) target distribution. We report this metric alongside the number of timesteps used by the Euler discretization.
For the image-generation experiments, we adopt the standard evaluation protocol and measure sample quality using the Fréchet Inception Distance (FID) as a function of the number of function evaluations (NFE).
For the MNIST small-model diagnostic, where Inception features are less natural for handwritten digits, we instead report LeNet-FID using a LeNet feature space.
When experiments are conducted within the EDM pipeline \citep{karras2022elucidating}, all components other than the timestep schedule are held fixed; so differences in log FID–NFE curves can be attributed solely to the choice of time discretization.
We use FID for image-generation  because it is the standard metric in the EDM and EDM2 evaluations and is applicable across the face, animal-face, and ImageNet experiments considered here.
Metrics such as Inception Score can be useful as supplemental diagnostics on class-diverse ImageNet-style data, but they are less informative for narrow-domain data sets such as FFHQ or AFHQv2 and are therefore not used here.

\emph{Training cost and amortization.}
ART-RL requires an offline one-off training stage to learn the schedule, whereas Uniform, EDM, and DPM are hand-designed and therefore training-free.
This training cost should be interpreted differently from inference or sampling cost: in our CIFAR--10 image experiment, learning the schedule for a given number of timesteps takes about 1--2 hours on a Colab T4 GPU, and this cost is paid only {\it once}.
After training, we find it justified to {\it distill} the learned policy into a fixed precomputed time grid; so deployment for sampling is identical to using a hand-designed schedule: the sampler simply reads a list of timesteps, and ART-RL introduces no additional inference-time overhead relative to EDM or DPM.
No retraining or architectural modification of the score model is required, nor is any change  made to the solver other than the locations at which it evaluates the same reverse dynamics.
In this sense, ART-RL is not a competing diffusion backbone or a new sampler implementation; it is a learned schedule that can be dropped into an existing sampler.
Moreover, we further experiment on amortizing the cost  by {\it transfer learning}. Specifically,
we reuse the {\it same} distilled schedule trained on CIFAR--10 with a certain number of timesteps across {\it different} timestep counts, target data sets, and the EDM2 latent-space pipeline without retraining, and find that the results still improve over those of the hand-designed grids.
Thus, the relevant practical question is not only whether ART-RL improves a single trained configuration, but whether a learned time parametrization can serve as a reusable schedule across many sampling settings. The generalization experiments reported in Section \ref{sec:generalization} will test this point.

\emph{Presentation of quantitative curves.}
Numerical results will be reported in various tables throughout this section; here we first present a visual overview in Figure~\ref{fig:quant-overview}.
To cater for different error scales,
the figure collects the error curves using a logarithmic vertical axis and a linear horizontal axis in timestep count or NFE.
This log-scale presentation avoids compressing the larger-budget regime, where FID, LeNet-FID, and Wasserstein errors are small but the differences between schedules remain important. Clearly, ART achieves the best results over all the experimented  data sets and timestep budgets.

\begin{figure}[!tb]
  \centering
  \includegraphics[width=\linewidth]{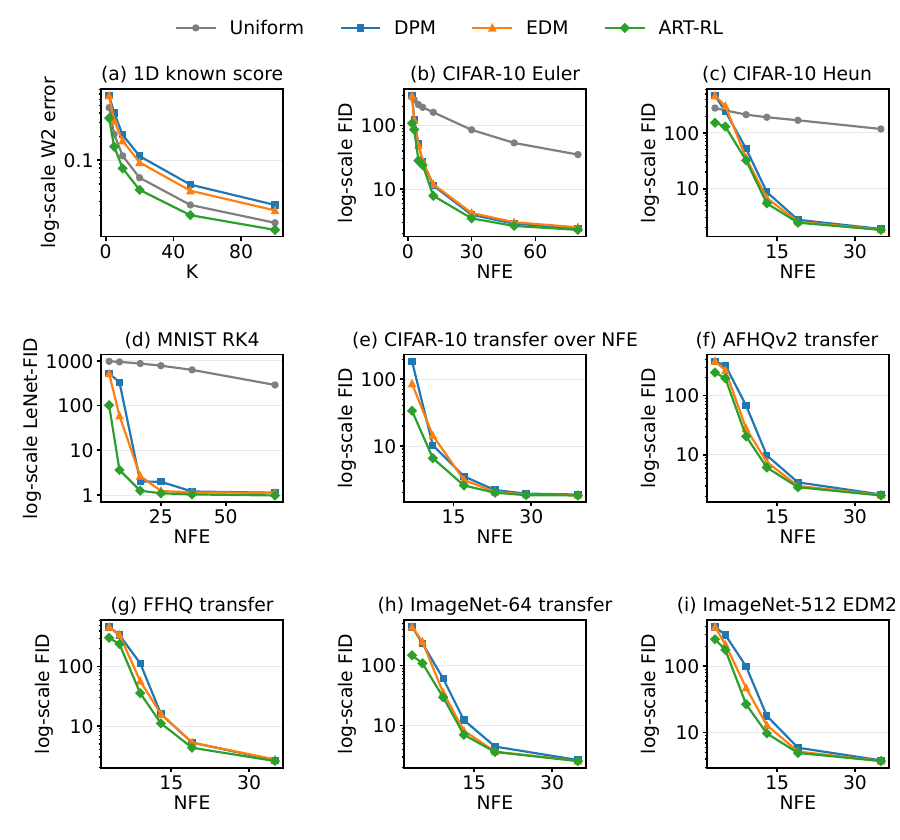}
  \caption{Visual overview of ART-RL across experiments. Each panel uses a logarithmic vertical axis and compares schedules at matched timestep counts or matched NFE. DPM denotes the DPM-Solver timestep grid.}
  \label{fig:quant-overview}
\end{figure}

\subsection{Experiment with Known Score Function}\label{subsec:oned}

The first experiment disentangles the effect of timestep allocation from score approximation, where a one-dimensional diffusion model for which all coefficients  in the probability-flow ODE are available in closed forms.
This setting allows sampling performance to be attributed solely to the choice of time discretization.

The forward diffusion begins at $p_0=\mathcal{N}(0,1)$, follows the variance exploding (VE) dynamics $\mathrm{d}x(t)=\sqrt{2t}\,\mathrm{d}w(t)$ and terminates at $T=3$.
The marginal law admits the explicit form $x(t)\sim p_t=\mathcal{N}(0,1+t^2)$, with the terminal distribution $p_T=\mathcal{N}(0,10)$.
The associated score function is then given by $S(t,x)=-x/(1+t^2)$.

Under this specialization, substituting the analytical score into the general definitions \eqref{eq:F-def} and \eqref{eq:Qexplicit} produces explicit expressions for the reparameterized probability-flow field and the Euler error coefficient,
namely $F(x,\psi)=-(T-\psi)x/\bigl(1+(T-\psi)^2\bigr)$ and
$Q(x,\psi)=x/\bigl(1+(T-\psi)^2\bigr)^2$.

We next examine the time-warping control learned by ART-RL. We test different number $K$ of timesteps, and here we discuss the case of $K=100$.
After training with $K=100$, we collect the realized $\theta$ sequences from the final $10{,}000$ backward trajectories.
To remove the residual terminal mismatch of individual rollouts, each trajectory is rescaled so that the resulting time change integrates exactly to $T$.
We then compute pointwise summary statistics across trajectories, reporting the empirical mean together with the interquartile (IQR) range (25--75 percentiles) at each timestep.
These aggregated statistics are shown in Figure~\ref{fig:distillation}.

\begin{figure}[tbp]
  \centering
  \includegraphics[width=0.7\linewidth]{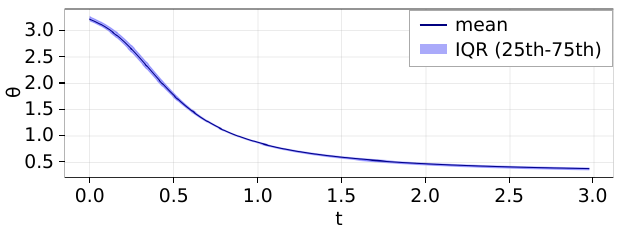}
  \caption{Empirical mean (solid line) and 25--75 percent IQR range (shaded region) of the executed control \(\theta\) across the last \(10{,}000\) training trajectories in the one-dimensional experiment with \(K=100\) timesteps. Each trajectory is normalized so that the resulting terminal time satisfies \(\psi(T) = T\).}
  \label{fig:distillation}
\end{figure}

Figure~\ref{fig:distillation} shows that the mean curve of \(\theta\) is very smooth and the IQR band is extremely narrow.
Moreover, when we plot the 99 percent empirical confidence band (see Appendix~\ref{app_subsec:oned}, Figure~\ref{fig:theta_mean_CI99}), the shaded region is visually indistinguishable from the mean curve. This observation is prevalent with other values of $K$, 
which indicates that, in this one-dimensional example, the learned control \(\theta\) depends only weakly on the state and can be effectively regarded as a deterministic function of time.
In other words, the policy has collapsed to an almost {\it time-only} schedule.

Motivated by this observation, we perform a simple {\it distillation} step for this\newline one-dimensional example: for each given $K$ we discard the neural network parameterization of the actor (which in general is a feedback policy function of $(t,x,\psi)$) and replace it with the empirical mean curve of \(\theta\) as a {\it fixed} function of \(t\).
This distilled approach has two important advantages.

First, it removes entirely the cost of evaluating a neural network to obtain \(\theta\) at each step.
Although the actor network is indeed not large (e.g. much smaller than the score model), repeatedly evaluating it along every sampling trajectory still incurs a nontrivial computational overhead.
After distillation, sampling under the ART-RL schedule requires no additional computation beyond that of standard schemes such as Uniform or EDM. In fact, the timestep sequence is precomputed once and then reused.

Second, it eliminates residual mismatch in the terminal time.
While the learned actor attempts to enforce \(\psi(T) = T\), individual trajectories may slightly overshoot or undershoot \(T\) when \(\theta\) is produced by a neural network at every timestep.
This discrepancy is negligible when the number of timesteps \(K\) is small but becomes significant as \(K\) grows due to the need of a finer time grid.
By distilling to a deterministic schedule whose increments are explicitly normalized to sum to \(T\), we guarantee that the induced time grid hits \(T\) exactly at the end, thereby improving the numerical fidelity of the discretized probability flow ODE.


We now present the results with different timestep budgets, where the comparison is remarkably consistent.
As shown in Figure~\ref{fig:quant-overview}(a) and Table~\ref{tab:1dim}, both DPM and EDM underperform the Uniform grid throughout.
By contrast, ART-RL achieves the best Wasserstein--2 error consistently, with a clear margin over all baselines.
These results also show that the hand-designed schedules DPM and EDM are designed {\it specifically} for image benchmarks and may fail in other domains even in simple toy examples. 

\begin{table}[tbp]
\centering
\begin{tabular}{lrrrrrr}
\toprule
{$K$} &    2   &    5   &    10  &    20  &    50  &    100 \\
\midrule
Uniform        & .468 & .215 & .114 & .060 & .027 & .016 \\
DPM            & .664 & .401 & .211 & .113 & .049 & .027 \\
EDM            & .664 & .319 & .177 & .094 & .041 & .023 \\
ART-RL         & \textbf{.345} & \textbf{.149} & \textbf{.079} & \textbf{.042} & \textbf{.020} & \textbf{.013} \\
\bottomrule
\end{tabular}
\caption{Wasserstein--2 error versus number of timesteps \(K\) in the one-dimensional experiment.}
\label{tab:1dim}
\end{table}

The one-dimensional study in this subsection 
isolates timestep effects using an analytical score model and shows that ART-RL can learn an effective schedule in a principled way.
Starting from the next subsection,  we move to image benchmarks and ask whether similar distillation and other generalization techniques still work empirically for more complex tasks.

\subsection{CIFAR--10 Under EDM Pipeline}\label{subsec:cifar}

We next conduct CIFAR--10 experiments within the official EDM pipeline \citep{karras2022elucidating}, keeping the score network, noise conditioning, hyperparameters, and all implementation details fixed across methods under comparison, except the timestep schedule.

First of all, for ART-RL trained on CIFAR--10 with $K=18$, Figure~\ref{fig:cifar10_theta_CI99} shows that the empirical 99 percent confidence band remains narrow around a smooth positive mean curve. Similar concentration is observed for other time step counts. 
This supports distilling the CIFAR--10 trained policies into deterministic time-only grids as in the one-dimensional example. In the experiments below, ART-RL is trained and distilled separately for each $K$. 

\begin{figure}[tbp]
\centering
\includegraphics[width=0.62\linewidth]{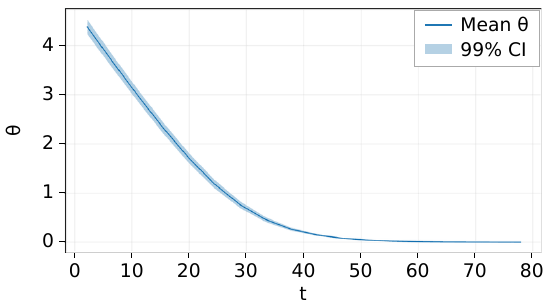}
\caption{Empirical mean of the executed control \(\theta\) and its 99 percent confidence interval for ART-RL trained on CIFAR--10 with $K=18$.} 
\label{fig:cifar10_theta_CI99}
\end{figure}

\subsubsection{Heun Sampling}

We first evaluate the distilled ART-RL schedules under the sampling configuration that is most relevant in modern image-generation pipelines.
In particular, the default EDM pipeline uses Heun, instead of Euler, and higher-order solvers of this type are widely adopted in practice because they provide improved accuracy per function evaluation.
Accordingly, we now evaluate Uniform, DPM, EDM, and ART-RL schedules on CIFAR--10 under the EDM pipeline with the Heun-based sampler, keeping all other components fixed so that differences in FID reflect only the effect of the timestep schedule.

We consider step counts $K\in\{2,3,5,7,10,18\}$ and follow the EDM implementation choice of using an Euler step for the final update.
With Heun updates, each intermediate step requires two score evaluations; so the overall cost is $\mathrm{NFE}=2K-1$.
In particular, $K=18$ corresponds to the strongest configuration reported by EDM for CIFAR--10 and is included as a budget-matched comparison.

Figure~\ref{fig:quant-overview}(c) and Table~\ref{tab:cifar10} show that ART-RL consistently achieves the best FID across all tested budgets. The improvement is especially pronounced at small to moderate NFEs, where ART-RL outperforms all hand-designed baselines by clear margins. Among the latter, DPM is slightly better than EDM at NFE \(=5\), while EDM becomes better from NFE \(=9\) onward. Both, however, remain consistently worse than ART-RL throughout. The outperformance of ART-RL  persists even at the largest budget, which is also the strongest configuration reported by \citet{karras2022elucidating}: at NFE \(=35\), ART-RL achieves 1.82 versus 1.85 for EDM. The robustness of the result at NFE \(=35\) is further supported by three additional matched runs (with 50,000 samples each), which give FIDs \(1.82, 1.79, 1.82\) for ART-RL versus \(1.85, 1.83, 1.85\) for EDM respectively.

\begin{table}[tbp]
\centering
\begin{tabular}{lrrrrrr}
\toprule
NFE &     3  &     5  &     9  &     13 &     19 &     35 \\
\midrule
Uniform & 280.29 & 254.47 & 213.13 & 191.69 & 168.87 & 118.02 \\
DPM     & 465.83 & 244.50 &  52.29 &   8.67 &   2.76 &   1.89 \\
EDM     & 465.83 & 305.15 &  35.54 &   6.79 &   2.54 &   1.85 \\
ART-RL  & \textbf{152.86} & \textbf{130.48} & \textbf{32.13} &   \textbf{5.44} &   \textbf{2.45} &   \textbf{1.82} \\
\bottomrule
\end{tabular}
\caption{FID versus number of function evaluations (NFE) on CIFAR--10 under Heun updates in EDM pipeline.}
\label{tab:cifar10}
\end{table}

Visual samples provided in Figure~\ref{fig:cifar10-three-in-a-row} show consistent results.
Uniform schedules produce visibly blurrier images even at larger budgets, while EDM and ART-RL generate sharp samples once sufficient numbers of evaluations are available.
At the smallest budgets (NFE$=3,5$), ART-RL already produces recognizable images, whereas EDM outputs remain closer to noise.

\begin{figure}[tbp]
  \centering
  \includegraphics[width=\linewidth]{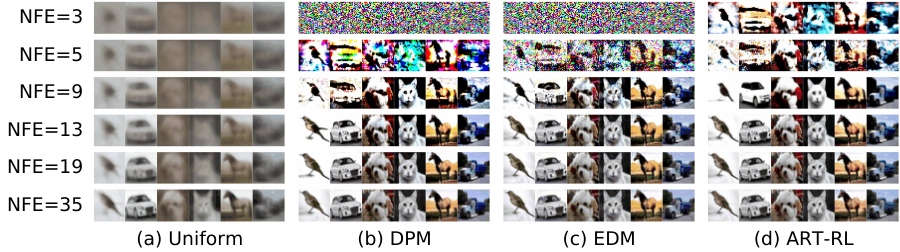}
  \caption{CIFAR--10 samples across timesteps for the four schedules (Uniform, DPM, EDM, ART-RL). Each panel shows a \(6\times 6\) grid where rows correspond to increasing NFE.}
  \label{fig:cifar10-three-in-a-row}
\end{figure}

Under Heun sampling, these results show that ART-RL can be deployed as a drop-in replacement for the EDM time grid within a competitive image-sampling pipeline.
It improves the sampler substantially in low- and mid-computation regimes while preserving strong performance at larger budgets.

\subsubsection{Euler Sampling}\label{subsec:cifar_euler}

To complement the Heun results, we also evaluate the distilled ART-RL schedules under Euler updates within the same EDM pipeline, thereby matching the Euler discretization used during ART-RL training.
As in the one-dimensional experiment, sampling proceeds with $K$ discrete Euler updates. In this case each update requires one score evaluation; so $\mathrm{NFE}=K$.
We evaluate representative step budgets $K\in\{2,3,5,7,12,30,50,80\}$.

Figure~\ref{fig:quant-overview}(b) and Table~\ref{tab:cifar10_euler_ablation} show that ART-RL consistently achieves the best FID across all tested budgets, with clear margins over all the baselines.
Under Euler updates, EDM is slightly better than DPM at smaller budgets, while DPM becomes slightly better at larger budgets. However, both are consistently worse than ART-RL.
This shows that the benefit of learning the time grid persists in a high-dimensional image model where the score is learned from data, even when ART-RL is used only by replacing the grid at which the same EDM reverse dynamics are evaluated.
Additional \(6\times 6\) visual results for the Euler ablation are provided in Appendix~\ref{app_subsec:cifar10-euler}, Figure~\ref{fig:cifar10-euler-three-in-a-row}.

\begin{table}[tbp]
\centering
\begin{tabular}{lrrrrrrrr}
\toprule
NFE & 2 & 3 & 5 & 7 & 12 & 30 & 50 & 80 \\
\midrule
Uniform & 280.50 & 255.02 & 214.60 & 194.40 & 162.14 & 85.83 & 53.40 & 34.99 \\
DPM     & 295.65 & 125.67 & 51.73  & 27.07 & 11.35 & 3.95 & 2.86 & 2.41 \\
EDM     & 295.65 & 122.56 & 49.10 & 27.73 & 11.91 & 4.21 & 3.01 & 2.50 \\
ART-RL  & \textbf{109.11} & \textbf{86.84} & \textbf{28.16} & \textbf{23.88} & \textbf{7.84} & \textbf{3.46} & \textbf{2.63} & \textbf{2.28} \\
\bottomrule
\end{tabular}
\caption{FID versus number of function evaluations (NFE) on CIFAR--10 under Euler updates in EDM pipeline.}
\label{tab:cifar10_euler_ablation}
\end{table}

\subsection{MNIST with Small Score Model}\label{subsec:mnist}

We next consider MNIST -- a deliberately simple setting in which the score network is lightweight (approximately $4.5$ MB), trained from scratch, and much less optimized than the pretrained image-generation models used in the EDM experiments.
This experiment is not part of the EDM pipeline and is not meant as a test of transfer across data sets.
Instead, it asks whether schedule learning remains useful when the score model itself is small and less accurate instead of large and optimally pretrained.
Such compact score models are common in latency- or memory-limited deployments; so it is important to understand whether timestep adaptation continues to improve sampling quality in this regime.

Again, to isolate the effect of the time grid, all the methods use exactly the same small score network and the same numerical integrator. Moreover, in this experiment 
we use the RK4 setting (see, e.g., \citealt{WCW24}) to test whether the learned schedule remains effective under an ODE solver that has an order higher than Heun.
Following the EDM/Heun sampling convention, we use RK4 for the non-terminal updates and retain an Euler step for the final update, which gives \(\mathrm{NFE}=4K-3\).
We then compare Uniform, DPM, EDM, and ART-RL schedules under identical training and sampling configurations.
In short, the timestep schedule is the only component that differs across the methods under comparison.

Figure~\ref{fig:quant-overview}(d) and Table~\ref{tab:mnist} show that ART-RL achieves the lowest LeNet-FID at every reported evaluation budget.
The advantage is consistent across the full NFE range, indicating that the learned timestep allocation remains effective when the score model is compact.
Figure~\ref{fig:mnist-three-in-a-row} in Appendix~\ref{app_subsec:mnist-qualitative}  further illustrates that ART-RL produces coherent digit samples earlier than the hand-crafted schedules.

\begin{table}[tbp]
\centering
\begin{tabular}{lrrrrrr}
\toprule
NFE & 5 & 9 & 17 & 25 & 37 & 69 \\
\midrule
Uniform & 981.13 & 953.74 & 876.58 & 783.65 & 632.26 & 290.46 \\
DPM & 523.60 & 334.44 & 2.01 & 1.97 & 1.20 & 1.12 \\
EDM & 523.60 & 59.36 & 2.66 & 1.23 & 1.14 & 1.12 \\
ART-RL & \textbf{102.13} & \textbf{3.62} & \textbf{1.25} & \textbf{1.09} & \textbf{1.03} & \textbf{0.98} \\
\bottomrule
\end{tabular}
\caption{LeNet-FID versus number of function evaluations (NFE) on MNIST.}
\label{tab:mnist}
\end{table}

The MNIST experiment indicates that the advantage of ART-RL extends beyond modern large-model pipelines.
Indeed,  here the gain over EDM and DPM is even larger, suggesting that learned timestep allocation remains effective even with simple, less optimized score-models.

\subsection{Transfer and Amortization of ART-RL}\label{sec:generalization}

We next experiment on the transferability of the ART-RL time schedule beyond the exact configuration in which it is learned.
This question is central to the {\it practical} value of ART-RL: if the learned schedule had to be retrained for every step count, data set, or sampling pipeline, the offline training cost would be harder to justify.
So we study whether  the same CIFAR--10 schedule can be amortized across a broad family of settings, turning the learned policy into a fixed plug-in timestep grid.
In each transfer experiment, ART-RL leaves the pretrained model, backbone, solver, and implementation pipeline unchanged and replaces only the hand-designed timestep schedule.
Thus, this section tests the strongest practical form of the drop-in claim: a schedule learned once in the CIFAR--10 EDM setting with a given time step budget is inserted directly into other budgets, data sets and even the EDM2 pipeline, and still outperforms. 

The study has three parts.
The first part focuses on flexibility within the same data set: starting from the CIFAR--10 schedule learned at \(K=18\) in Section~\ref{subsec:cifar}, we construct schedules for other step counts via interpolation and extrapolation.
The second part tests transfer across data sets within the pixel-space EDM pipeline: we reuse the same CIFAR--10 schedule on AFHQv2, FFHQ, and ImageNet--64 without any additional training.
The third part further tests transfer to the EDM2 pipeline \citep{karras2024analyzing} on ImageNet--512.
This setting changes three aspects at once: it uses a more modern EDM2 backbone and sampling pipeline, operates in latent space rather than pixel space, and evaluates high-resolution image generation.
Together, these experiments assess whether the learned time parametrization captures structure that persists not only across time budgets and data sets, but also across pipelines, resolutions, and representation spaces.

In this subsection, we restrict attention to DPM, EDM, and ART-RL.
The Uniform grid is substantially worse in the corresponding image settings and is therefore omitted.

\subsubsection{Transfer Across Timestep Counts on CIFAR--10}\label{sec:interp}
We first examine whether the ART-RL schedule learned at $K=18$ can be reused at other step counts.
All the experiments follow the same CIFAR--10 EDM-pipeline configuration as in Section~\ref{subsec:cifar}.
For ART-RL, we take the learned \(K=18\) sampling grid and generate new grids for \(K'\in\{4,6,9,12,15,20\}\) by log-linear resampling of the remaining-time values \(T-\psi\). Specifically, let \(0=\psi_0^{(K)}<\cdots<\psi_K^{(K)}=T\) denote the learned \(K\)-step grid, where the superscript indicates the step count. For \(j=0,\ldots,K'-1\), set \(r_j=j(K-1)/K'\), \(i_j=\lfloor r_j\rfloor\), and \(\alpha_j=r_j-i_j\), and define \(\psi_j^{(K')}=T-\exp\{(1-\alpha_j)\log(T-\psi_{i_j}^{(K)})+\alpha_j\log(T-\psi_{i_j+1}^{(K)})\}\), with \(\psi_{K'}^{(K')}=T\). This is log-linear in \(T-\psi\), and the same rule is used for interpolation to smaller step counts \(K'<K\) and extrapolation to larger step counts \(K'>K\).
For EDM, the timestep sequence at each $K$ is computed directly from its analytic rule.

Figure~\ref{fig:quant-overview}(e) and Table~\ref{tab:CIFAR-10-OtherN} show that the \(K=18\) ART-RL schedule transfers smoothly across different timestep counts. ART-RL still  achieves the best FID for all the reported NFEs, outperforming both EDM and DPM after interpolation and extrapolation of the learned schedule. This suggests that the learned time parametrization captures a stable allocation pattern that remains effective under changes in grid resolution.

\begin{table}[tbp]
\centering
\begin{tabular}{lrrrrrr}
\toprule
NFE & 7 & 11 & 17 & 23 & 29 & 39 \\
\midrule
DPM    & 185.63 & 10.31 & 3.52 & 2.19 & 1.94 & 1.88 \\
EDM    & 85.80 & 14.42 & 3.11 & 2.06 & 1.88 & 1.85 \\
ART-RL & \textbf{33.73} & \textbf{6.59} & \textbf{2.57} & \textbf{2.00} & \textbf{1.84} & \textbf{1.82} \\
\bottomrule
\end{tabular}
\caption{FID versus number of function evaluations (NFE) on CIFAR--10 for interpolated and extrapolated timestep counts. }
\label{tab:CIFAR-10-OtherN}
\end{table}

Additional $6\times 6$ image grids for these interpolated and extrapolated schedules are provided in Appendix \ref{app_subsec:cifar10-interp}, Figure~\ref{fig:cifar10-interp-three-in-a-row}.

\begin{table}[!b]
\centering
\begin{tabular}{llrrrrrr}
\toprule
Data Set & Method & 3 & 5 & 9 & 13 & 19 & 35 \\
\midrule
\multirow{3}{*}{AFHQv2}
& DPM    & 375.76 & 321.59 & 67.64 & 9.77 & 3.44 & 2.15 \\
& EDM    & 375.76 & 266.02 & 27.88 & 7.56 & 2.99 & 2.11 \\
& ART-RL & \textbf{243.48} & \textbf{194.79} & \textbf{20.48} & \textbf{6.12} & \textbf{2.85} & \textbf{2.07} \\
\midrule
\multirow{3}{*}{FFHQ}
& DPM    & 466.76 & 340.51 & 113.87 & 15.94 & 5.25 & 2.66 \\
& EDM    & 466.76 & 344.76 & 57.13 & 15.87 & 5.26 & 2.73 \\
& ART-RL & \textbf{305.97} & \textbf{240.38} & \textbf{35.73} & \textbf{11.08} & \textbf{4.31} & \textbf{2.57} \\
\midrule
\multirow{3}{*}{ImageNet--64}
& DPM    & 437.42 & 233.35 & 60.48 & 12.31 & 4.46 & 2.66 \\
& EDM    & 437.42 & 248.32 & 35.32 & 8.18 & 3.68 & 2.57 \\
& ART-RL & \textbf{147.21} & \textbf{108.47} & \textbf{29.49} & \textbf{7.01} & \textbf{3.62} & \textbf{2.56} \\
\bottomrule
\end{tabular}
\caption{FID versus number of function evaluations (NFE) for transfer across data sets. The ART-RL schedule is learned on CIFAR--10 and reused without retraining. }
\label{tab:cross_dataset_transfer}
\end{table}

\subsubsection{Cross-Data-Set Transfer to AFHQv2, FFHQ, and ImageNet--64}\label{sec:transfer}

We next test transfer across data sets without retraining.
For each target data set and time step budget, we keep the corresponding EDM pipeline unchanged, including the score network, solver configuration, and all hyperparameters, and replace only the timestep grid by the ART-RL schedule learned on CIFAR--10 in Section~\ref{subsec:cifar}.
The hand-designed baselines use their corresponding schedules at the same step counts, and the NFE accounting follows the same convention as in the CIFAR--10 experiments.

Figure~\ref{fig:quant-overview}(f)--(h), together with Table~\ref{tab:cross_dataset_transfer}, show that the learned CIFAR--10 schedule transfers successfully  to all the three data sets.
ART-RL achieves the lowest FID in every reported setting, suggesting that the learned time parametrization 
remains effective as a drop-in grid replacement across different image distributions under the same EDM pipeline.

\subsubsection{Transfer to EDM2 on ImageNet--512}\label{sec:edm2_transfer}

\begin{table}[!t]
\centering
\begin{tabular}{lrrrrrr}
\toprule
NFE & 3 & 5 & 9 & 13 & 19 & 35 \\
\midrule
DPM    & 392.19 & 297.26 & 99.38 & 17.86 & 5.92 & 3.82 \\
EDM    & 392.19 & 213.45 & 47.33 & 12.91 & 5.19 & 3.74 \\
ART-RL & \textbf{256.13} & \textbf{176.50} & \textbf{26.78} & \textbf{9.73} & \textbf{4.94} & \textbf{3.73} \\
\bottomrule
\end{tabular}
\caption{FID versus number of function evaluations (NFE) on ImageNet--512 under the EDM2 pipeline using the XS model. }
\label{tab:imagenet512_edm2}
\end{table}

\begin{figure}[!b]
  \centering
  \includegraphics[width=\linewidth]{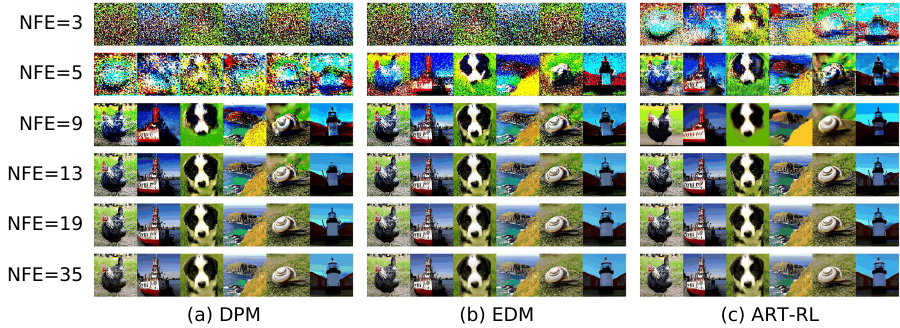}
  \caption{ImageNet--512 samples under the EDM2 pipeline for the three schedules (DPM, EDM, ART-RL). Each panel shows samples at increasing NFEs.}
  \label{fig:imagenet512-edm2-three-in-a-row}
\end{figure}

We further evaluate whether the learned ART-RL schedule transfers beyond the pixel-space EDM pipeline.
Specifically, we test ImageNet--512 under the EDM2 pipeline, which uses a latent-space diffusion model rather than directly operating in pixel space.
We use the extra-small (XS) EDM2 ImageNet--512 model.
As before, we keep the score model, solver configuration, and all hyperparameters fixed, and compare only the timestep schedules.

Figure~\ref{fig:quant-overview}(i) and Table~\ref{tab:imagenet512_edm2} show that ART-RL continues to outperform both EDM and DPM under the EDM2 pipeline.
The same conclusion holds also under Inception Score: ART-RL attains higher scores than both EDM and DPM at every reported budget; see Appendix~\ref{app_subsec:imagenet512-is}, Table~\ref{tab:imagenet512_edm2_is}.
This provides the strongest transfer test in our study: the schedule learned from CIFAR--10 under the EDM pipeline remains effective when moved to a different backbone, a different sampling pipeline, a latent representation, and a substantially higher image resolution.
Importantly, this EDM2 experiment still changes only the timestep grid; the pretrained EDM2 model and the rest of the sampling pipeline are left intact.
Thus, this ImageNet--512 experiment answers affirmatively whether ART-RL remains useful in a modern high-resolution diffusion setting.


\begin{samepage}
\section{Conclusion}
\label{sc6}

This paper introduces and develops ART, a control-theoretic framework,  for timestep allocation in score-based diffusion sampling. While ART features a deterministic optimal control problem, solving it via the conventional HJB equations is insurmountable due to the typical (ultra) high dimensions of the state space with most generative AI tasks. We remedy the problem by introducing ART-RL, which is an auxiliary problem with a particular class of Gaussian policies, and establishing a precise relationship between the two problems. The ART-RL problem can be solved algorithmically {\it \`a la} the recently developed continuous-time reinforcement learning theory including policy evaluation and policy improvement. It is noteworthy that 
here ART-RL serves as a technical device to solve ART, rather than a prescription for exploration due to model uncertainty.

\end{samepage}

\begin{samepage}
Existing time allocation schedules are mostly hand-crafted and {\it ad hoc} to the underlying tasks. As a result, they do not necessarily work across different types of tasks.  For instance, it is shown in this paper that EDM, the state-of-the-art pipeline for image generation, works poorly for a very simple one-dimensional example with a known score. 
By contrast,  ART is a principled approach premised upon a rigorous  theory encompassing general  diffusion generative jobs. This is validated by our empirical study: 
the learned time schedules improve sample quality at matched evaluation budgets across different tasks and numerical solvers, including Euler, Heun, and RK4.
They also generalize across timestep counts, data sets, sampling pipelines, and representation spaces. 
In particular, the experimentally demonstrated transferability of ART has  important practical implications.  Once a schedule is learned and distilled in image generation, ART  has the same inference-time form as EDM or DPM; hence its offline training cost is amortized across many downstream sampling settings.

\end{samepage}

Remarks on possible future directions are in order. 
Our analysis is restricted to probability flow ODE sampling, and extending the formulation to SDE samplers may lead to different allocation behaviors and new theoretical questions.
The current objective is motivated by an Euler local error surrogate; it would be interesting to investigate alternative criteria, including surrogates aligned with higher-order integrators, to better connect the control principle to practical solvers.
Finally, while distillation to time-only schedules is effective in our experiments, it removes state dependence as a consequence, and it is not yet clear if/when richer state-conditioned schedules may provide additional benefits.
Overall, ART and ART-RL offer a first step toward a systematic, theory-grounded design of timestep schedules for diffusion-based generative modeling.

\acks{%
This work was presented at the 2026 IMSI Workshop on Foundations of Multi-Agent and Mean Field Reinforcement Learning in Chicago, the 2026 Asian Quantitative Finance Conference in Pohang, the 2026 INFORMS Annual Meeting in San Francisco, the 2025 Workshop on Stochastic Control, Financial Technology, and Machine Learning in Hong Kong, the 2026 Quantitative Finance Conference in Singapore, the 2026 Workshop on Numerical Analysis in Artificial Intelligence and MWORKS Practice in Guangzhou, the 2026 International Conference on Optimal Control and Machine Learning in Guiyang, the 2026 Shandong--Hong Kong--Macao Forum on Mathematics and Its Interdisciplinary Applications in Jinan, and at a seminar at the University of Southern California. We are grateful to the participants at these events for their comments.

Yilie Huang acknowledges financial support from the Start-up Fund of The Hong Kong Polytechnic University (Project ID: P0063874).
Wenpin Tang is supported by NSF CAREER Award DMS-2538791, NSF MFAI Award DMS-2602038, and the Tang
Family Assistant Professorship.
Xun Yu Zhou is supported by the Nie Center for Intelligent Asset Management at Columbia University.
The authors declare no competing interests.
}

\clearpage
\appendix
\section{Monotone Time Reparameterization}\label{app:monotone-art}

In the main text, the control clock speed $\theta$ is allowed to be negative, so the time reparameterization can be non-monotone. This section treats the case where the control is constrained by $\theta(t)\geq0$ for almost every $t\in[0,T]$. The main idea, inspired by the approach of \citet{li2026dynamic} to a portfolio selection problem with a no-shorting constraint, is to use Gaussian policies with nonnegative means rather than Gaussian policies truncated to the nonnegative half-line.

Fix the time-budget multiplier $\gamma$ and take $\lambda>0$. In both ART and ART-RL, the state starts at $(x(s),\psi(s))=(y,\phi)$. Write $\E_{s,y,\phi}$ for conditional expectation and $q_t=|Q(x(t),\psi(t))|$ along the respective trajectory. As in the main text, assume $q_t>0$, well-posed dynamics, and finite objectives.

\noindent\emph{ART with nonnegative controls.}
Let \mbox{$\mathcal U_+=\{\theta:\theta(t,x,\psi)\ge0\ \forall(t,x,\psi)\}$} be the class of nonnegative ART feedbacks. With $\theta_t=\theta(t,x(t),\psi(t))$, the problem is
\begin{align}
\dot x(t)&=\theta_t F(x(t),\psi(t)),\qquad \dot\psi(t)=\theta_t,\label{eq:monotone-art-dynamics}\\
J^\theta(s,y,\phi)&=\E_{s,y,\phi}\!\left[\gamma T-\int_s^T(q_t\theta_t^2+\gamma\theta_t)\,\dd t\right],\notag\\
V_+(s,y,\phi)&=\sup_{\theta\in\mathcal U_+}J^\theta(s,y,\phi).\label{eq:monotone-art-value}
\end{align}

\noindent\emph{ART-RL with nonnegative Gaussian means.}
Let $\Pi_+^{(\lambda)}=\{\pi_\mu^{(\lambda)}\in\Pi^{(\lambda)}:\mu\in\mathcal U_+\}$, where $\pi_\mu^{(\lambda)}$ is the Gaussian policy in \eqref{eq:gaussian_policy} with mean $\mu$.
Write $\pi_{\mu,t}^{(\lambda)}=\pi_\mu^{(\lambda)}(\cdot\mid t,x(t),\psi(t))$. The constrained ART-RL problem is
\begin{align}
\dot x(t)&=\int_{\mathbb R}uF(x(t),\psi(t))\,\pi_{\mu,t}^{(\lambda)}(\dd u),\qquad
\dot\psi(t)=\int_{\mathbb R}u\,\pi_{\mu,t}^{(\lambda)}(\dd u),\label{eq:monotone-rl-dynamics}\\
J^{\pi_\mu^{(\lambda)}}(s,y,\phi)&=\E_{s,y,\phi}\!\left[(\gamma+\lambda)T
-\int_s^T\!\int_{\mathbb R}(q_tu^2+\gamma u)\,\pi_{\mu,t}^{(\lambda)}(\dd u)\,\dd t\right],\notag\\
V_+^{(\lambda)}(s,y,\phi)&=\sup_{\pi_\mu^{(\lambda)}\in\Pi_+^{(\lambda)}}J^{\pi_\mu^{(\lambda)}}(s,y,\phi).\label{eq:monotone-rl-value}
\end{align}

The HJB equations for the constrained ART and ART-RL problems are
\begin{align}
(V_+)_t+\sup_{u\ge0}\big\{\big((V_+)_x^\top F+(V_+)_\psi-\gamma\big)u-|Q|u^2\big\}
&=0,\label{eq:monotone-art-hjb}\\
(V_+^{(\lambda)})_t+\sup_{u\ge0}\big\{\big((V_+^{(\lambda)})_x^\top F+(V_+^{(\lambda)})_\psi-\gamma\big)u-|Q|u^2\big\}-\lambda
&=0,\label{eq:monotone-rl-hjb}
\end{align}
with terminal conditions $V_+(T,x,\psi)=\gamma T$ and $V_+^{(\lambda)}(T,x,\psi)=(\gamma+\lambda)T$, respectively.

\begin{theorem}\label{thm:monotone-equivalence}
If $V_+$ is a classical solution of the HJB equation~\eqref{eq:monotone-art-hjb}, then
\[
V_+^{(\lambda)}(t,x,\psi)=V_+(t,x,\psi)+\lambda t
\]
is a classical solution of the HJB equation~\eqref{eq:monotone-rl-hjb}. Moreover,
\begin{equation}\label{eq:monotone-optimal-control}
\begin{gathered}
\pi_{\mu_+^*}^{(\lambda)}(\cdot\mid t,x,\psi)
=\mathcal N\!\left(\mu_+^*(t,x,\psi),\frac{\lambda}{|Q(x,\psi)|}\right),\\
\text{with}\qquad
\mu_+^*(t,x,\psi)=
\frac{\big[(V_+)_x^\top F(x,\psi)+(V_+)_\psi-\gamma\big]_+}{2|Q(x,\psi)|},
\qquad [a]_+=\max\{a,0\},
\end{gathered}
\end{equation}
is the optimal policy for the auxiliary problem~\eqref{eq:monotone-rl-value} subject to the~dynamics~\eqref{eq:monotone-rl-dynamics}. Finally, $\mu_+^*(t,x,\psi)$ is the optimal feedback for the original problem~\eqref{eq:monotone-art-value} subject to the~dynamics~\eqref{eq:monotone-art-dynamics}.
\end{theorem}

\begin{proof}
The proof is analogous to that of Theorem~\ref{thm_relationship}. At each state the Gaussian moments are
\[
\int_{\mathbb R}u\,\pi_\mu^{(\lambda)}(\dd u)=\mu,
\qquad
\int_{\mathbb R}u^2\,\pi_\mu^{(\lambda)}(\dd u)=\mu^2+\frac{\lambda}{|Q|}.
\]
The first identity gives the same dynamics as the deterministic control $\theta=\mu$. By the second identity and the terminal rewards, we have
\[
J^{\pi_\mu^{(\lambda)}}(s,y,\phi)-J^\mu(s,y,\phi)=\lambda T-\lambda(T-s)=\lambda s.
\]
Taking suprema gives the value shift. Since the shift is independent of $\mu$, optimality transfers in both directions. The same argument gives the two-way correspondence in the original unconstrained formulations.

Write $A=(V_+)_x^\top F+(V_+)_\psi-\gamma$. The shift $\lambda t$ adds $\lambda$ to the time derivative, leaves the state derivatives unchanged, and shifts the terminal value from $\gamma T$ to $(\gamma+\lambda)T$. Thus the two HJB equations reduce to the same maximization of $Au-|Q|u^2$ over $u\ge0$, with maximizer \eqref{eq:monotone-optimal-control}. Applying the chain rule along the averaged trajectory for $\pi_\mu^{(\lambda)}$ and using the terminal condition give
\[
\begin{aligned}
&J^{\pi_\mu^{(\lambda)}}(s,y,\phi)-V_+^{(\lambda)}(s,y,\phi)\\
&\qquad=\E_{s,y,\phi}\!\int_s^T
\big[(V_+^{(\lambda)})_t-\lambda+A\mu-|Q|\mu^2\big]\,\dd t\ \le\ 0.
\end{aligned}
\]
The inequality follows from the ART-RL HJB equation, with equality for the maximizing feedback $\mu_+^*$. Hence $\pi_{\mu_+^*}^{(\lambda)}$ is optimal for constrained ART-RL, and the objective shift gives optimality of $\mu_+^*$ for the constrained ART problem.
\end{proof}

\section{Additional Numerical Results}\label{app:numerical}

This appendix presents implementation details and additional numerical results.

\subsection{Implementation Details}\label{app:implementation}

The actor and critic are parameterized by multilayer perceptrons and trained with Adam using $(\beta_1,\beta_2)=(0.9,0.999)$ and zero weight decay. A reference configuration for the single-trajectory algorithm uses a learning rate of $10^{-4}$ for both networks and a multiplier step size of $10^{-4}$. In Algorithm~\ref{alg:art-rl-actor-critic}, each iteration simulates a backward trajectory using $K$ Euler steps of size $\Delta t=T/K$ on the reparameterized clock, with Gaussian variance $\lambda/(|Q|\vee\varepsilon)$ at each step. The trajectory supplies the discrete moment updates for both networks and the terminal discrepancy $\psi_n(T)-T$ for the multiplier update.

For $Q$ in \eqref{eq:Qexplicit}, the spatial derivative enters as the Jacobian--vector product $(D_x\hat S)F$, and the time derivative is taken with respect to the scalar input $T-\psi$ at fixed $x$. This avoids forming the full Jacobian. For the distillation in Section~\ref{subsec:oned}, we first rescale each rollout's time increments to sum to $T$, then average the rescaled increments at each grid point. Sampling uses the resulting precomputed grid without evaluating the actor, critic, or $Q$.

\subsection{Additional ImageNet--512 Inception Score Results}\label{app_subsec:imagenet512-is}

Table~\ref{tab:imagenet512_edm2_is} reports the Inception Score for the ImageNet--512 EDM2 experiment.

\begin{table}[!bp]
\centering
\begin{tabular}{lrrrrrr}
\toprule
NFE & 3 & 5 & 9 & 13 & 19 & 35 \\
\midrule
DPM    & 1.58 & 1.82 & 20.17 & 106.78 & 175.94 & 205.48 \\
EDM    & 1.58 & 4.19 & 53.10 & 129.75 & 186.40 & 206.49 \\
ART-RL & \textbf{4.15} & \textbf{7.50} & \textbf{79.68} & \textbf{147.40} & \textbf{188.28} & \textbf{207.37} \\
\bottomrule
\end{tabular}
\caption{Inception Score versus NFE on ImageNet--512 under the EDM2 pipeline using the XS model. }
\label{tab:imagenet512_edm2_is}
\end{table}

\subsection{Direct Deterministic Grid Optimization}\label{app:dgo}

For a given NFE budget, direct grid optimization (DGO) optimizes the interior points of the sampling grid. Both DGO and ART-RL target the ART Euler-error surrogate under the same total-time budget. DGO starts from uniform increments and uses Adam to optimize the positive step sizes, which sum to $T$. ART-RL learns a feedback policy through the randomized formulation and actor--critic updates, then distills it into a fixed grid. We match the rollout budget, CIFAR--10 model, EDM pipeline, Heun solver, and evaluation protocol.

\begin{table}[!bp]
\centering
\begin{tabular}{lrrrrr}
\toprule
NFE & 3 & 9 & 13 & 19 & 35 \\
\midrule
EDM & 465.83 & 35.54 & 6.79 & 2.54 & 1.85 \\
DGO & 183.68 & 33.71 & 6.27 & 2.51 & 1.84 \\
ART-RL & \textbf{152.86} & \textbf{32.13} & \textbf{5.44} & \textbf{2.45} & \textbf{1.82} \\
\bottomrule
\end{tabular}
\caption{CIFAR--10 FID for direct grid optimization and ART-RL.}
\label{tab:dgo}
\end{table}

In Table~\ref{tab:dgo}, DGO improves on EDM at all five budgets, showing the benefit of the ART objective even under direct optimization. ART-RL further lowers FID at all five budgets. This supports the actor--critic approach under the matched protocol.

\section{Additional Images}

This appendix collects supplementary images from the experiments.

\subsection{Results for the One-Dimensional Study}\label{app_subsec:oned}
Figure~\ref{fig:theta_mean_CI99} shows the empirical mean of the executed control \(\theta\) together with the 99 percent confidence band computed from the last \(10{,}000\) trajectories in the one-dimensional experiment reported in Section~\ref{subsec:oned}.
As in the main text, each trajectory is normalized so that the induced terminal time satisfies \(\psi(T)=T\).
The confidence band is extremely narrow and visually indistinguishable from the mean curve, confirming that in this setting the learned control exhibits negligible variability across trajectories and can be treated as an effectively deterministic function of time.

\begin{figure}[tbp]
  \centering
  \includegraphics[width=0.7\linewidth]{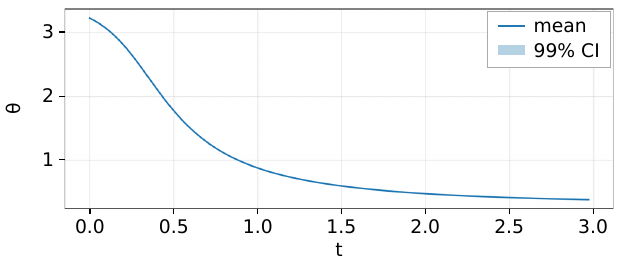}
  \caption{Empirical mean of the executed control \(\theta\) and its 99 percent confidence interval, based on the last \(10{,}000\) trajectories in the one-dimensional experiment.}
  \label{fig:theta_mean_CI99}
\end{figure}

\subsection{Additional Qualitative Results for CIFAR--10}\label{app_subsec:cifar10}

This appendix collects qualitative grids for CIFAR--10 that complement the quantitative results in Section~\ref{subsec:cifar_euler} (Euler ablation) and Section~\ref{sec:interp} (interpolation and extrapolation across timestep counts).

\subsubsection{Euler Ablation Grids}\label{app_subsec:cifar10-euler}
Figure~\ref{fig:cifar10-euler-three-in-a-row} shows qualitative results for the Euler ablation in Section~\ref{subsec:cifar_euler}, comparing Uniform, EDM, and ART-RL under the same Euler-based sampler.
The ART-RL schedule tends to produce recognizable structure earlier at small budgets, consistent with the FID values reported in the main text.

\begin{figure}[!tbp]
  \centering
  \includegraphics[width=\linewidth]{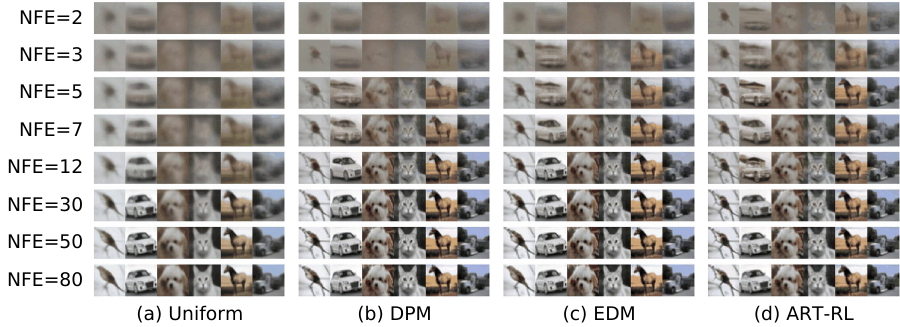}
  \caption{CIFAR--10 samples across evaluation budgets under Euler updates. Each panel shows an $8\times 6$ grid where rows correspond to increasing NFE.}
  \label{fig:cifar10-euler-three-in-a-row}
\end{figure}

\subsubsection{Interpolation and Extrapolation Grids}\label{app_subsec:cifar10-interp}
Figure~\ref{fig:cifar10-interp-three-in-a-row} reports qualitative results for the interpolation and extrapolation study in Section~\ref{sec:interp}.
We reuse the schedule learned at $K=18$ and construct grids for other timestep counts via log-linear interpolation and extrapolation, while DPM and EDM are computed from their analytic rules at each $K$.

\begin{figure}[tbp]
  \centering
  \includegraphics[width=\linewidth]{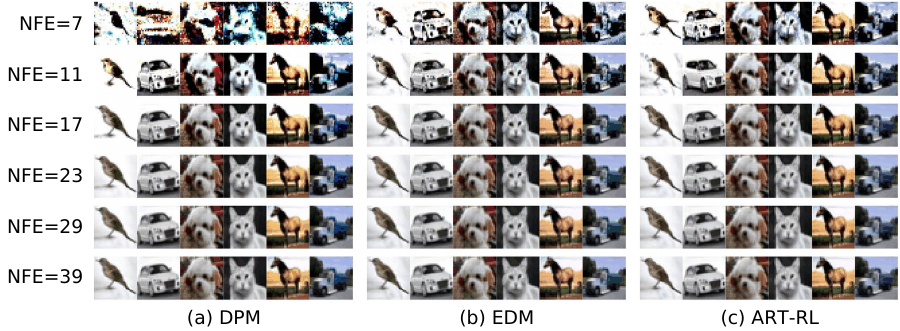}
  \caption{CIFAR--10 samples across evaluation budgets for interpolated and extrapolated timestep counts. Each panel shows a $6\times 6$ grid where rows correspond to increasing NFE.}
  \label{fig:cifar10-interp-three-in-a-row}
\end{figure}

\subsection{Additional Qualitative Results for MNIST}\label{app_subsec:mnist-qualitative}

Figure~\ref{fig:mnist-three-in-a-row} provides qualitative grids for the MNIST small-model experiment in Section~\ref{subsec:mnist}.
The DPM grid is included together with Uniform, EDM, and ART-RL so that the visual comparison matches the quantitative table in the main text.

\begin{figure}[tbp]
  \centering
  \includegraphics[width=\linewidth]{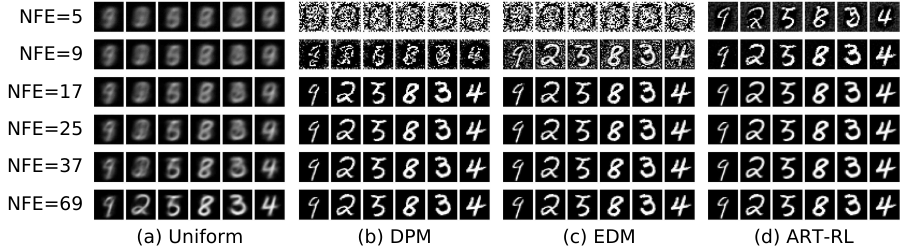}
  \caption{MNIST samples across timesteps for the four schedules (Uniform, DPM, EDM, ART-RL). Each panel shows a \(6\times 6\) grid where rows correspond to increasing NFE.}
  \label{fig:mnist-three-in-a-row}
\end{figure}

\subsection{Additional Qualitative Transfer Results}\label{app_subsec:transfer-grids}

Figures~\ref{fig:afhqv2-three-in-a-row}--\ref{fig:imagenet-three-in-a-row} provide qualitative grids for the experiments on transfer across data sets in Section~\ref{sec:transfer}.

\begin{figure}[tbp]
  \centering
  \includegraphics[width=\linewidth]{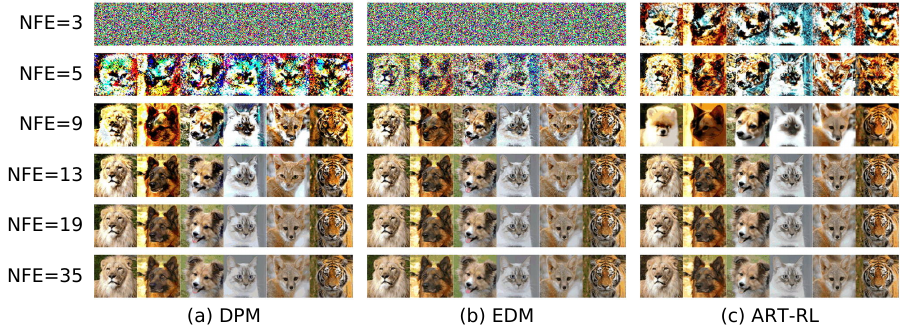}
  \caption{AFHQv2 samples across timesteps for the three schedules (DPM, EDM, ART-RL). Each panel shows a \(6\times 6\) grid where rows correspond to increasing NFE.}
  \label{fig:afhqv2-three-in-a-row}
\end{figure}

\begin{figure}[tbp]
  \centering
  \includegraphics[width=\linewidth]{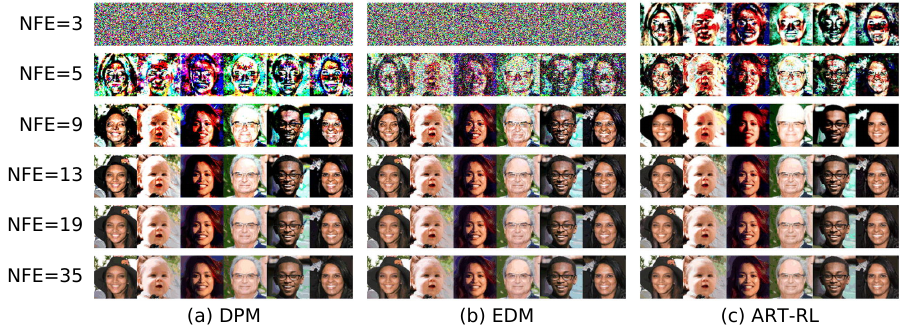}
  \caption{FFHQ samples across timesteps for the three schedules (DPM, EDM, ART-RL). Each panel shows a \(6\times 6\) grid where rows correspond to increasing NFE.}
  \label{fig:ffhq-three-in-a-row}
\end{figure}

\begin{figure}[tbp]
  \centering
  \includegraphics[width=\linewidth]{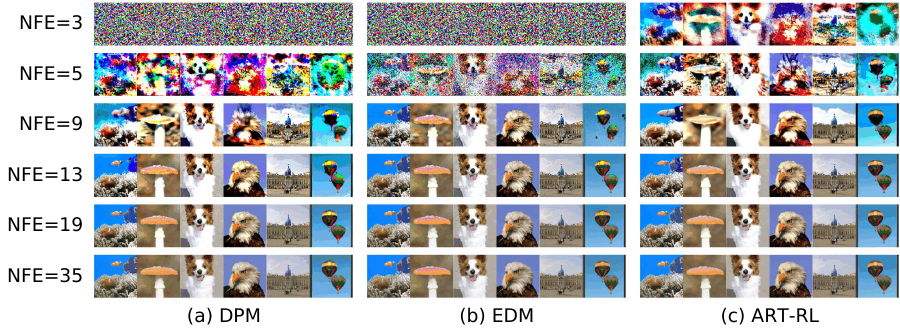}
  \caption{ImageNet--64 samples across timesteps for the three schedules (DPM, EDM, ART-RL). Each panel shows a \(6\times 6\) grid where rows correspond to increasing NFE.}
  \label{fig:imagenet-three-in-a-row}
\end{figure}

\clearpage
\bibliography{ref}

\end{document}